\documentclass{article}

\PassOptionsToPackage{colorlinks=true,citecolor=blue,linkcolor=blue,urlcolor=blue}{hyperref}

\usepackage[preprint]{icml2026}
\usepackage{hyperref}

\usepackage[T1]{fontenc}
\usepackage[utf8]{inputenc}

\usepackage{microtype}
\usepackage{graphicx}
\usepackage{subcaption}
\usepackage{booktabs} 
\usepackage{wrapfig}
\usepackage{multirow}
\usepackage{nicefrac}
\usepackage{dsfont}
\usepackage[dvipsnames]{xcolor}
\usepackage{enumitem}
\usepackage{tikz}
\usepackage[most]{tcolorbox}
\usepackage{listings}

\usepackage{amsmath}
\usepackage{amssymb}
\usepackage{amsfonts}
\usepackage{mathtools}
\usepackage{amsthm}

\makeatletter
\@ifpackageloaded{natbib}{\setcitestyle{apa,square}}{}
\makeatother

\usepackage[nameinlink]{cleveref}

\crefname{theorem}{Theorem}{Theorems}
\Crefname{theorem}{Theorem}{Theorems}
\crefname{lemma}{Lemma}{Lemmas}
\Crefname{lemma}{Lemma}{Lemmas}

\makeatletter

\expandafter\let\csname algorithm*\endcsname\relax
\expandafter\let\csname endalgorithm*\endcsname\relax

\makeatother

\usepackage[ruled,boxed,noend,linesnumbered]{algorithm2e}
\SetKwInOut{Input}{Input}
\SetKwInOut{Init}{Initialization}
\SetKwInOut{Output}{Output}
\DontPrintSemicolon


\crefname{algocf}{algorithm}{algorithms}
\Crefname{algocf}{Algorithm}{Algorithms}
\crefname{AlgoLine}{line}{lines}
\Crefname{AlgoLine}{Line}{Lines}



\theoremstyle{plain}
\newtheorem{theorem}{Theorem}[section]

\newtheorem{lemma}[theorem]{Lemma}

\theoremstyle{definition}

\theoremstyle{remark}

\usepackage[textsize=tiny]{todonotes}
\definecolor{mybluebg}{HTML}{CBDCEC}
\definecolor{myblueborder}{HTML}{689CC8}
\tcbset{
  modelbox/.style={
    enhanced, breakable,
    colback=mybluebg, colframe=myblueborder,
    boxrule=0.8pt, arc=2pt,
    left=6pt, right=6pt, top=6pt, bottom=6pt,
    fonttitle=\bfseries, title={#1}
  }
}

\newcommand{\RedDown}[1]{\textcolor{red!70!black}{\scriptsize$\downarrow$#1}}
\newcommand{\GreenDown}[1]{\textcolor{green!50!black}{\scriptsize$\downarrow$#1}}

\newcommand{\cellann}[3][\deltavshift]{%
  \begin{tikzpicture}[baseline=(n.base)]
    \node[inner sep=0pt, outer sep=0pt] (n) {#2\/};
    \path[use as bounding box] (n.south west) rectangle (n.north east);
    \node[anchor=north west, inner sep=0pt, xshift=.15em, yshift=#1]
         at (n.north east) {\scriptsize #3};
  \end{tikzpicture}%
}

\definecolor{dkgreen}{rgb}{0,0.6,0}
\definecolor{gray}{rgb}{0.5,0.5,0.5}
\definecolor{mauve}{rgb}{0.58,0,0.82}
\definecolor{lightblue}{rgb}{0.9,0.95,1.0}
\definecolor{lightgray}{rgb}{0.9,0.9,0.9}
\let\hat\widehat
\let\tilde\widetilde

\let\oldcheck\check
\newcommand{\widecheck}[1]{\oldcheck{#1}} 
\let\check\widecheck

\newcommand{\norm}[1]{\lVert#1\rVert}
\newcommand{\abs}[1]{\lvert#1\rvert}

\newcommand{\defeq}{\coloneqq}
\newcommand{\argmax}{\operatorname*{arg\,max}}

\newcommand{\EE}{\mathbb{E}}

\newcommand{\cD}{\mathcal{D}}

\newcommand{\cL}{\mathcal{L}}

\icmltitlerunning{LACONIC: Length-Aware Constrained Reinforcement Learning for LLM}

\begin{document}

\twocolumn[
  \icmltitle{LACONIC: Length-Aware Constrained Reinforcement Learning for LLM}



  \icmlsetsymbol{equal}{*}

  \begin{icmlauthorlist}
    \icmlauthor{Chang Liu}{ucla}
    \icmlauthor{Yiran Zhao}{salesforce}
    \icmlauthor{Lawrence Liu}{ucla}
    \icmlauthor{Yaoqi Ye}{nus}
    \icmlauthor{Csaba Szepesv\'{a}ri}{alberta}
    \icmlauthor{Lin F. Yang}{ucla}
  \end{icmlauthorlist}

  \icmlaffiliation{ucla}{University of California, Los Angeles}
  \icmlaffiliation{salesforce}{Salesforce AI Research}
  \icmlaffiliation{nus}{National University of Singapore}
  \icmlaffiliation{alberta}{University of Alberta}

  \icmlcorrespondingauthor{Chang Liu}{changliu11@ucla.edu}

  \icmlkeywords{Machine Learning, ICML}

  \vskip 0.3in
]



\printAffiliationsAndNotice{}  

\begin{abstract}

     Reinforcement learning (RL) has enhanced the capabilities of large language models (LLMs) through reward-driven training. Nevertheless, this process can introduce excessively long responses, inflating inference latency and computational overhead. Prior length-control approaches typically rely on fixed heuristic reward shaping, which can misalign with the task objective and require brittle tuning. In this work, we propose \texttt{LACONIC}, a reinforcement learning method that enforces a target token budget during training. Specifically, we update policy models using an augmented objective that combines the task reward with a length-based cost. To balance brevity and task performance, the cost scale is adaptively adjusted throughout training. This yields robust length control while preserving task reward. We provide a theoretical guarantee that support the method. Across mathematical reasoning models and datasets, \texttt{LACONIC} preserves or improves \texttt{pass@1} while reducing output length by over 50\%. It maintains out-of-domain performance on general knowledge and multilingual benchmarks with 44\% fewer tokens. Moreover, \texttt{LACONIC} integrates into standard RL-tuning with no inference changes and minimal deployment overhead.
 
\end{abstract}

\section{Introduction}

\begin{figure}[!b]
  \centering
  \includegraphics[width=\columnwidth]{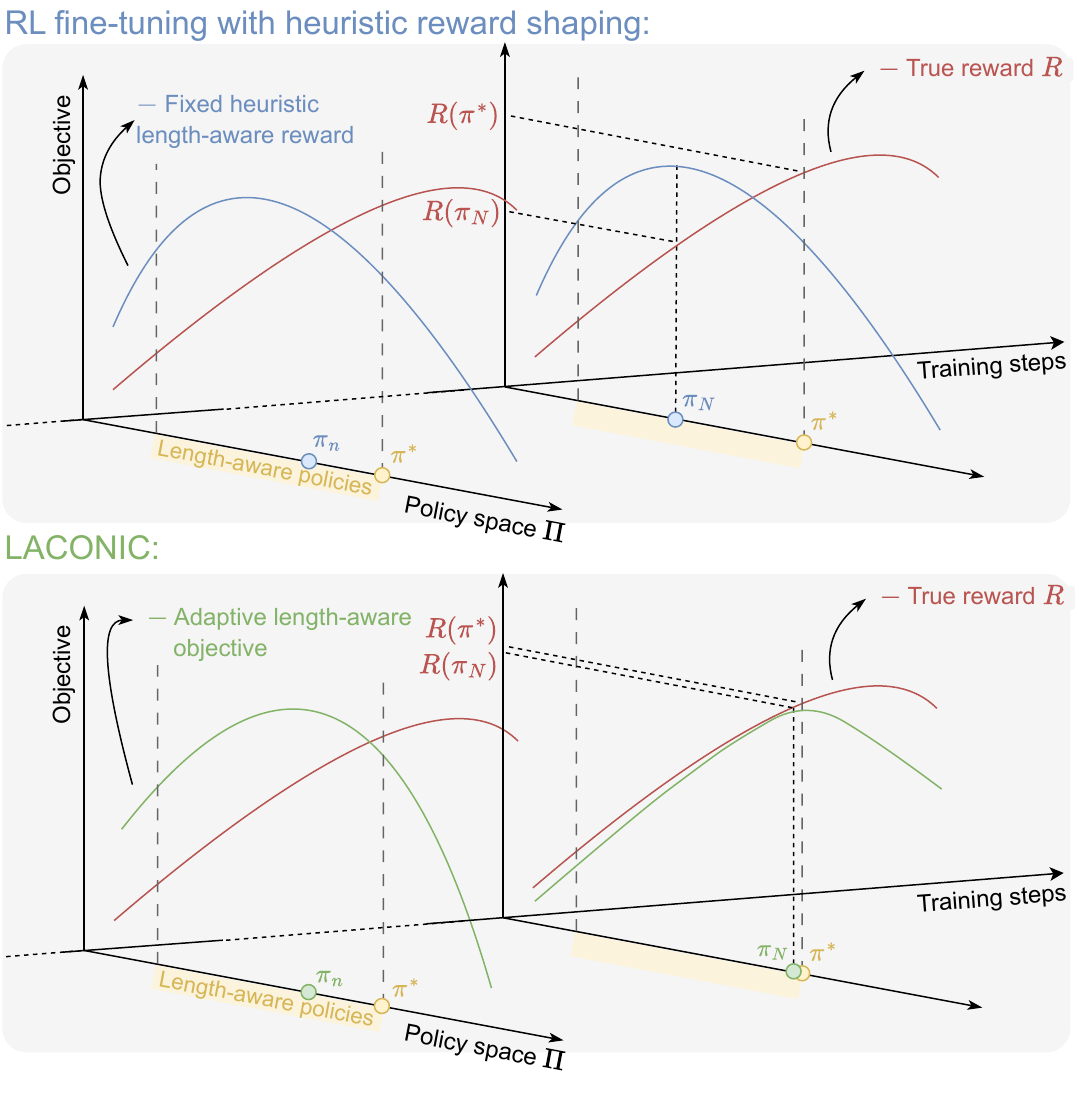}
  \caption{The top panel skethces RL-tuning with a fixed length-aware shaping objective (blue). As the heuristically shaped objective generally differs from the true task reward $R$ (red), optimizing it may converge to a policy $\pi_N$ that is suboptimal in $R$. The bottom panel sketches training with \texttt{LACONIC}. \texttt{LACONIC} adaptively updates the length-aware objective (green) so that it better aligns with the true task reward while achieving shorter outputs, yielding near-optimal policies.}
  \label{fig:illustration_1}
\end{figure}

\begin{figure*}[t]
    \includegraphics[width=\textwidth]{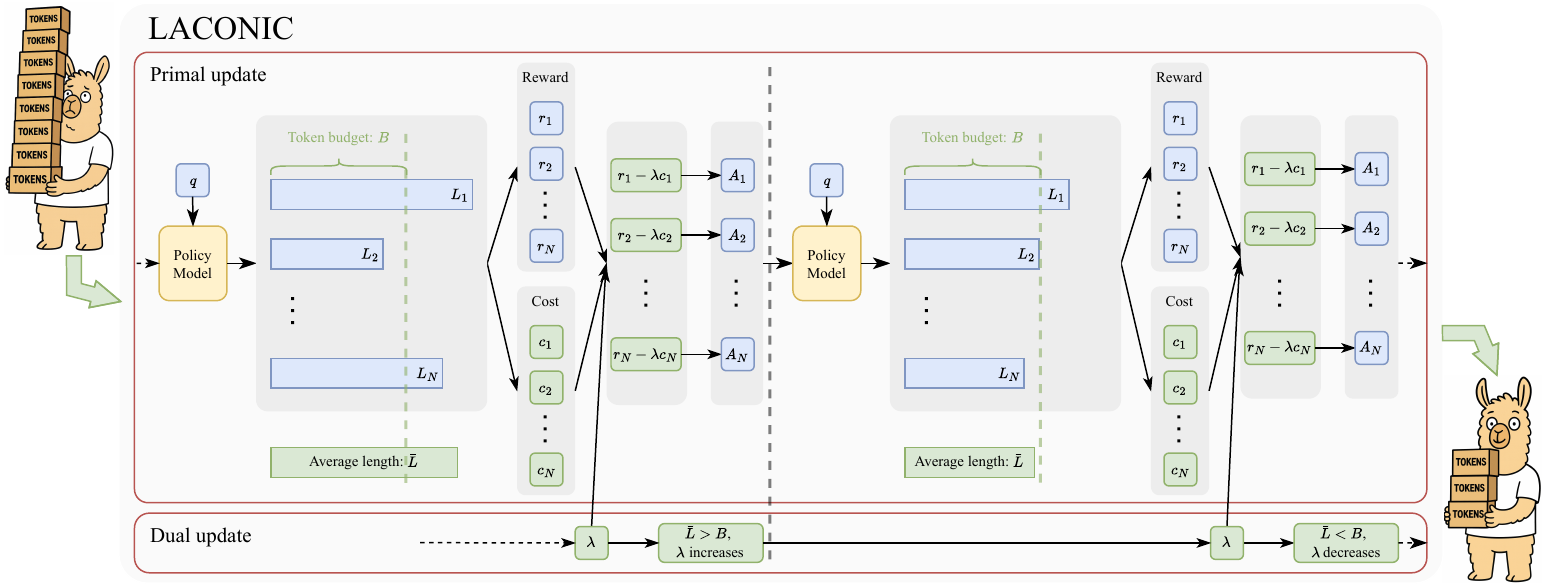}
    \caption{Illustration of LACONIC. LACONIC alternates two steps: (1) in a primal update, the policy model is updated on an augmented objective that trades off task reward $r$ with a length-aware cost $c$ scaled by the dual variable $\lambda$; (2) in a dual update, $\lambda$ is adaptively updated to enforce a token budget constraint $B$ by increasing when the average length $\bar{L}$ exceeds the budget $B$ and decreasing otherwise. Together, these updates maximize task reward while meeting the budget on average.}\vspace{-0.5cm}\label{fig:laconic}
\end{figure*}

Large language models (LLMs) such as GPT \citep{openai2024openaio1card, gpt5}, Gemini \citep{comanici2025gemini}, DeepSeek \citep{deepseekai2025deepseekr1incentivizingreasoningcapability}, and Claude\citep{claude-4.1} have witnessed unprecedented success in its applications from software agents to enterprise analytics \citep{team2025kimi, li2025webthinker, jin2025search, feng2025retool}. The impressive capabilities of LLMs have been significantly enhanced by reinforcement learning based fine-tuning \citep{li2025search, wu2025agentic, li2025webthinker}, a procedure that aligns pretrained models with task-specific rewards through interaction with an environment. This process has been pivotal in refining LLM reasoning skills, enhancing generalization, and achieving state-of-the-art performance across diverse benchmarks \citep{wang2024qimprovingmultistepreasoning, hsiao2025a, shi2025toollearningwildempowering, Qu2025}. However, RL-tuned language models often suffer from generating unnecessarily long thinking traces. This problem is particularly acute on reasoning and mathematics tasks, where the model is asked to spell out logical steps \citep{chen2025think23overthinkingo1like,sui2025stopoverthinkingsurveyefficient}.  In practice, excessive verbosity inflates training and inference time, increases memory pressure, and ultimately degrades user experience. 

Recent work on length-aware LLMs has explored positional encoding, prompt engineering, and post-generation truncation \citep{li2025search, wu2025agentic, li2025webthinker}. A straightforward method is to design new reward functions to incorporate response length signals into RL-tuning \citep{aggarwal2025l1controllinglongreasoning,cheng2025incentivizingdualprocessthinking,huang2025hapotraininglanguagemodels,yuan2025efficientrltrainingreasoning}. These methods typically hard-code a length penalty or heuristic reward shaping that stays fixed throughout training. Fine-tuning with these rewards optimizes a surrogate objective that is misaligned with true task reward, and often demands per-setting hyperparameter tuning. The top panel in \cref{fig:illustration_1} visualizes this by showing a sketch of the training process. The fixed heuristic objective (blue curve) differs from the true objective (red curve), so optimizing it yields policies with suboptimal task rewards. 

In this paper, we address length control in RL-tuning by maximizing task reward subject to an average token budget constraint. We introduce \texttt{LACONIC} (\textbf{L}ength-\textbf{A}ware \textbf{Con}strained Pol\textbf{ic}y Optimization), a primal-dual algorithm. During training, the model samples candidate responses for each prompt. Besides the task reward (e.g., correctness or usefulness) as in the standard RL-tuning, we assign to each candidate response also a length-aware cost proportional to its budget violation. We then construct a learning signal that combines task reward with length cost, scaled by an adaptively learned multiplier. We alternatively update the policy model and the multiplier. The policy is updated by a policy optimization step where advantages and objectives are calculated from the constructed signal. Then we update the multiplier based on the average response length of the current batch. We raise the multiplier if the current batch violates the token budget constraint on average, and lower it if the current batch falls short. This feedback automatically steers the model's average output length towards the token budget. Particularly, when the model consistently stays within the token budget, the multiplier naturally drops to zero and our training steps reduce to standard RL-tuning steps, allowing the model to recover task rewards with shortened responses. As illustrated in \cref{fig:illustration_1}, \texttt{LACONIC} adopts an adaptive objective that dynamically align the optimization with true rewards, steering the policy towards the length-aware optimum.

\vspace{-0.1cm}


In \cref{subsec:theory-clipped}, we provide theoretical analysis of our approach. Under standard assumptions, we establish convergence guarantees and bound the resulting model’s near-optimality.

\vspace{-0.1cm}

We conduct extensive experiments to evaluate our method \texttt{LACONIC} and present the evaluation results in \cref{sec:experiment}. We apply \texttt{LACONIC} to fine-tune two reasoning models DeepScaleR-1.5B-Preview \citep{deepscaler2025} and DeepSeek-R1-Distill-Qwen-1.5B \citep{deepseekai2025deepseekr1}. The experimental results show that our \texttt{LACONIC}-tuned models can significantly outperform existing length control baselines L1 \citep{aggarwal2025l1controllinglongreasoning}, \cite{arora2025traininglanguagemodelsreason}, and ThinkPrune \citep{hou2025thinkprune} and mostly preserve the performance of the full-length base models on pass@1 across common mathematics benchmarks, while reducing response lengths by using fewer tokens. \texttt{LACONIC} also preserves accuracy on benchmarks outside our RL-tuning domain while substantially reducing response length.
Furthermore, we perform ablation analysis on the hyperparameters of our method in \cref{sec:analysis}, showing that \texttt{LACONIC} provides robust length control.

\section{Methodology}


\subsection{Preliminary Background}
RL fine-tuning casts text generation as a sequential decision process, where the prompt $q$ together with the partial output sequence constitutes the state, selecting the next token is the action, and the language model parameterized by $\theta$ serves as the policy $\pi_\theta$ that maps states to action probabilities. After generating the response, the model receives task rewards $r(q,o)$ assigned by a reward model. Policy gradient algorithms such as Proximal Policy Optimization (PPO) \citep{schulman2017proximalpolicyoptimizationalgorithms} and Group Relative Policy Optimization (GRPO) \citep{shao2024deepseekmathpushinglimitsmathematical} then translate such rewards into token-level gradients, so that the model can be updated to increase the corpus-level expected rewards, i.e, 
\begin{equation}\label{eq:RL-tuning objective}
    \max_\theta\EE_{q\sim P(Q), o\sim\pi_\theta(\cdot|q)}[r(q,o)].    
\end{equation}
In practice, GRPO updates the policy model's parameters $\theta$ by optimizing the following surrogate objective
\begin{equation}\label{eq:GRPO_obj}
\begin{aligned}
    \mathcal{J}&(\theta)
    =\mathbb{E}_{q,\{o_i\}_{i=1}^G}\!\Bigl[
    \frac{1}{G}\sum_{i=1}^G \frac{1}{|o_i|}\sum_{t=1}^{|o_i|}
    \min(\rho_{i,t}A_{i,t},\,\! \\
    &\operatorname{clip}(\rho_{i,t},1\!-\!\varepsilon,1\!+\!\varepsilon)A_{i,t})
    -\beta D_{\mathrm{KL}}[\pi_\theta\|\pi_{\mathrm{ref}}]\Bigr].
\end{aligned}
\end{equation}
where $\rho_{i,t} = \frac{\pi_{\theta}(o_{i,t}\mid q, o_{i,<t})}{\pi_{\theta_{\mathrm{old}}}(o_{i,t}\mid q, o_{i,<t})}$ is the likelihood ratio, and $A_{i,t}$ is the group-relative advantage. The likelihood ratio clipping and an extra KL-divergence penalty are adopted to stabilize policy updates.



\subsection{Formulation}\label{sec:formulation}


To explicitly control response lengths, we extend standard RL-tuning to a constrained setting that maximizes task reward under an average token constraint $B$, a pre-specified budget reflecting deployment targets such as latency and computational resources in practice. Formally, letting $L(o)$ denote the length of response $o$, we address
\begin{equation}\label{eq:CMDP}
\begin{aligned}
    &\max_\theta\EE_{q\sim P(Q), o\sim\pi_\theta(\cdot|q)}[r(q,o)], \\
    &\,\textrm{s.t. }\; \EE_{q\sim P(Q), o\sim\pi_\theta(\cdot|q)}[L(o)] \le B.
\end{aligned}
\end{equation}
In \cref{eq:CMDP}, we enforce a corpus-level average token budget rather than a strict per-sequence length constraint, as response lengths naturally vary across prompts (e.g., math olympiad problems typically require more tokens than simple arithmetic). This allows allocating more tokens to hard instances while still meeting the overall budget constraint.


A standard approach to \cref{eq:CMDP} is Lagrangian primal-dual optimization. Introduce a dual variable $\lambda\ge 0$ and Lagrangian
$\mathcal{L}(\theta,\lambda):= \mathbb{E}[r(q,o)]-\lambda\left(\mathbb{E}[L(o)]/B-1\right),$
then the constrained problem \cref{eq:CMDP} can be addressed by optimizing the Lagrangian over policy models $\theta$ while adjusting $\lambda\ge 0$, i.e., $\max_\theta\min_{\lambda\ge 0}\cL(\theta,\lambda)$. Deferring the derivation to \cref{app:ablation_cost}, we obtain the following primal and dual updates:
\begin{equation}\label{eq:ideal_primal}
    \theta_{t+1}\in \argmax_{\theta}\ \mathbb{E} \left[r(q,o)-\lambda_t\,\frac{L(o)-B}{B}\right],
\end{equation}
\begin{equation}\label{eq:ideal_dual}
    \lambda_{t+1} = \max\{\lambda_t+\eta\left(\mathbb{E}[L(o)]/B-1\right), 0\}.
\end{equation}
The primal update \cref{eq:ideal_primal} is analogous to standard RL-tuning in \cref{eq:RL-tuning objective}. It maximizes expected reward augmented by a linear length-aware cost $\tilde{c} \defeq (L(o) - B)/B$ weighted by $\lambda$. However, this linear cost $\tilde{c}$ can be problematic in practice: when $\lambda>0$, the linear cost $\tilde{c}$ consistently incentivizes shorter outputs, and drives the policy towards extremely short responses. This can cause unstable policy updates in training. We observe this behavior in practice and report the empirical results in \cref{app:ablation_cost}.




\subsection{\texttt{LACONIC}: Length-Aware Constrained Policy Optimization}\label{sec:laconic}
We propose \texttt{LACONIC}, a primal-dual method for length-aware RL tuning that preserves the principled constrained-optimization structure, while resolving the collapse to overly short outputs induced by the linear length cost $\tilde{c}$.

\textbf{Clipped cost.} We introduce a clipped cost that is zero up to the budget. Specifically, for a prompt $q$ and a response $o$ generated by the policy model $\pi_\theta$, we define
\begin{equation}\label{eq:cost}
    c(q,o) = \max\left\{\frac{L(o) - B}{B}, 0\right\}.
\end{equation}\vspace{-0.5cm}

This clipped cost assigns equal cost to within-budget responses ($L(o)\le B$). Therefore, it removes artificial pressure to shorten responses that already meet the token budget.

\begin{algorithm}[t]
\caption{\texttt{LACONIC} (Length-Aware Constrained Policy Optimization)}\label{algo}
\LinesNumbered
        \KwIn{initial policy model $\pi_{\theta_{\textrm{init}}}$; reward models $r_\varphi$; task prompts $\cD$; token budget $B$; step size $\eta$; initial dual variable $\lambda_{\textrm{init}}$}
        policy model $\pi_\theta\leftarrow \pi_{\theta_{\textrm{init}}}$\;
        dual variable $\lambda \leftarrow \lambda_{\textrm{init}}$\;
        \For{iteration = $1, \dots, I$}{
            reference model $\pi_{\textrm{ref}} \leftarrow \pi_\theta$\;
            \For{step = $1, \dots, M$}{
                Sample a batch $\cD_b$ from $\cD$\;
                Update the old policy model $\pi_{\theta_{\textrm{old}}}\leftarrow \pi_\theta$\;
                Sample $G$ outputs $\{o_i\}_{i=1}^G\sim \pi_{\theta_{\textrm{old}}}(\cdot|q)$ for each question $q\in \cD_b$\;
                Compute rewards $\{r_i\}_{i=1}^G$ for each sampled output $o_i$ by running $r_\varphi$\;
                Compute costs $\{c_i\}_{i=1}^G$ for each sampled output $o_i$ by \cref{eq:cost}\;
                Compute Lagrangian rewards $\{\ell_{\lambda,i}\}_{i=1}^G$ for each sample output by \cref{eq:Lag_reward}\;
                Compute advantages $A_{i,t}$ for the $t$-th token of $o_i$ by \cref{eq:advantage}\;
                \tcp{Primal update}
                Update the policy model $\pi_\theta$ by maximizing the GRPO-style objective \nllabel{line:primal} \;
                \tcp{Dual update}
                Update the dual variable $\lambda$ by \cref{eq:dual-update}\;
            }
        }
        \KwOut{$\pi_\theta$}
\end{algorithm}

\textbf{Primal updates.} In the primal update, we update the policy model $\theta$ by holding $\lambda_t$ fixed and solving \cref{eq:ideal_primal}. This objective has the same form as standard RL-tuning in \cref{eq:RL-tuning objective} with a shaped reward. Therefore, we instantiate the primal update using the RL-tuning procedure (e.g., GRPO) by substituting the task reward $r$ with the Lagrangian reward:
\begin{equation}\label{eq:Lag_reward}
    \ell_\lambda(q,o) = r(q,o) - \lambda\cdot c(q,o),
\end{equation}
where we replace the linear cost with our clipped cost $c$ to stabilize policy updates.

Specifically, for each prompt $q$, we sample a group of candidate outputs $\mathbf{o}=\{o_1,o_2,\dots,o_G\}$ from the current policy model $\pi_\theta$, and compute their task rewards $\mathbf{r}=\{r_1,r_2,\dots,r_G\}$ and costs $\mathbf{c} = \{c_1,c_2,\dots,c_G\}$ by \cref{eq:cost}. We then compute the Lagrangian rewards $\pmb{\ell}_\lambda = \{\ell_{\lambda,1},\ell_{\lambda,2},\dots,\ell_{\lambda,G}\}$ by \cref{eq:Lag_reward}. For each token $o_{i,t}$, we construct the GRPO-style advantage by normalizing Lagrangian rewards $\ell_{\lambda,i}$ within the group, i.e.,
\begin{equation}\label{eq:advantage}
    A_{i,t} = \tilde{\ell}_{\lambda,i} = \frac{\ell_{\lambda,i} - \mathrm{mean}(\pmb{\ell}_\lambda)}{\mathrm{std}(\pmb{\ell}_\lambda)}.
\end{equation}
The policy model is optimized by maximizing the GRPO objective in \cref{eq:GRPO_obj} where advantages are calculated by \cref{eq:advantage}. 

\textbf{Dual updates.} In the dual update, we adjust the dual variable $\lambda$ by estimating the expectation in \cref{eq:ideal_dual} with the empirical mean response length $\bar{L}$ of the current minibatch. This yields the practical dual update
\begin{equation}\label{eq:dual-update}
    \lambda \leftarrow \textrm{clip}\left(\lambda + \eta\left(\frac{\bar{L}}{B} - 1\right), 0, \Lambda\right),
\end{equation}
with the step size $\eta$ and a $\lambda$-ceiling $\Lambda$. When the batch violates the token budget on average ($\bar{L}>B$), the update increases $\lambda$, raising the effective price of tokens in $\ell_\lambda$. Longer responses then receive lower (often negative) advantages than shorter responses with similar task rewards, so the next primal update shifts the policy toward shorter outputs. When the batch falls within the budget, $\lambda$ relaxes towards 0. Notably, when $\lambda=0$, $\ell_\lambda$ reduces to the task reward $r$, and the next primal update is exactly a GRPO step. This feedback adapts $\lambda$ to track the budget constraint throughout training as the policy and length distribution evolve. 

We cap $\lambda$ by a ceiling $\Lambda$ to avoid over-penalizing length. For an indicator reward $r(q,o)\in\{0,1\}$, an excessively large $\lambda$ can make a within-budget incorrect response (with $\ell_\lambda=0$) score above a correct but long response. A sufficient safeguard is to require $\ell_\lambda(q,o_c) > 0$ for any correct response $o_c$, which requires $\lambda < \frac{B}{L(o_c)-B}$ for $L(o_c)>B$. Using the worst case $L(o_c) \le L_{\max}$, where $L_{\max}$ is the maximum response length cap, it is sufficient to set $\Lambda = \frac{B}{L_{\max}-B}$.


We present \texttt{LACONIC} in \cref{algo} and illustrate the workflow with two sample steps in \cref{fig:laconic}. 

\section{Theoretical Results}
\label{subsec:theory-clipped}
As discussed in \cref{sec:laconic}, \texttt{LACONIC} performs policy (primal) updates with a clipped cost $\max\{\frac{L}{B} - 1, 0\}$, while updating the dual variable $\lambda$ using the average linear cost $\frac{\bar{L}}{B} - 1$ in \cref{eq:dual-update}. We refer to this combination as the clipped-cost primal-dual update, and we provide the theoretical guarantee in this section.

\textbf{Notation.} Let $\EE_\pi[\cdot] = \EE_{q\sim P(Q),\,o\sim\pi(\cdot|q)}[\cdot]$, and $(x)_+ = \max\{x,0\}$. Define the shorthands $R(\pi)=\EE_\pi[r(q,o)]$, $\tilde C(\pi)=\EE_\pi\!\left[\tfrac{L(o)-B}{B}\right]$, and $C(\pi)=\EE_\pi\!\left[\tfrac{(L(o)-B)_+}{B}\right]$. Then \cref{eq:CMDP} can be written as $\max_\pi R(\pi)$ s.t.\ $\tilde C(\pi)\le 0$, and let $\pi^\star\in\argmax_{\pi:\tilde{C}(\pi)\le 0} R(\pi)$ be an optimal feasible policy.

\vspace{0.05cm}To isolate the effect of clipping, consider an idealized setting where expectations are exact and each primal step finds an exact maximizer.
The idealized clipped-cost primal-dual updates are
\begin{equation}\label{eq:ideal_dyn_primal}
    \pi_t \in \argmax_{\pi}\left\{R(\pi)-\lambda C(\pi)\right\},
\end{equation}
\begin{equation}\label{eq:ideal_dyn_dual}
\lambda_{t+1}=\textrm{clip}\left(\lambda_t+\eta\,\tilde{C}(\pi), 0, \Lambda\right).
\end{equation}
\Cref{app:clipped-convergence} shows that \cref{eq:ideal_dyn_primal,eq:ideal_dyn_dual} converge to a feasible limit pair $(\pi^\sharp,\lambda^\sharp)$ under mild regularity conditions.
Then \cref{thm:price_of_clipping} quantifies the reward gap between $\pi^\sharp$ and the optimal policy $\pi^\star$, with the proof deferred to \cref{app:clipped-convergence}.

\begin{theorem}[Price of clipped cost]\label{thm:price_of_clipping}
Let $\pi^\star\in\max_{\pi:\tilde{C}(\pi)\le 0} R(\pi)$ be an optimal feasible policy of the length-constrained problem in \cref{eq:CMDP}. Let $(\pi^\sharp,\lambda^\sharp)$ be the feasible limit of the idealized clipped-cost primal-dual updates in \cref{eq:ideal_dyn_primal,eq:ideal_dyn_dual}. Then
\begin{equation}\label{eq:price_of_clipping_prelim}
0 \le R(\pi^\star)-R(\pi^\sharp)\le \lambda^\sharp C(\pi^*).
\end{equation}
Moreover, for indicator rewards with the $\lambda$-ceiling $\Lambda = \frac{B}{L_{\max} - B}$ and a maximum length cap $L(o)\le L_{\max}$, we have
\begin{equation}\label{eq:price_of_clipping}
0 \le R(\pi^\star)-R(\pi^\sharp)\le \frac{B}{L_{\max}}.
\end{equation}
\end{theorem}
The bound indicates that the suboptimality induced by clipping is governed by two factors: the limiting multiplier
$\lambda^\sharp$, and the extent to which the optimal feasible policy $\pi^\star$ places probability mass above the budget, as captured by $C(\pi^\star)$. 
In the indicator-reward regime, combining the $\lambda$-ceiling with a worst-case tail bound of $\tilde{C}(\pi^*)$  yields the ceiling-based guarantee $B/L_{\max}$.
This bound is conservative and often loose in practice, but it already implies near-optimality in the worst-case sense whenever $L_{\max}\gg B$.

\begin{table*}[t]
\centering
\caption{Evaluation results across four math benchmarks.}
\renewcommand{\arraystretch}{1.2}
\resizebox{\textwidth}{!}{%
\begin{tabular}{lcccccccccc}
\toprule
\multirow{2}{*}{\textbf{Model}} &
\multicolumn{2}{c}{\textbf{AIME2024}} & 
\multicolumn{2}{c}{\textbf{MATH}} & 
\multicolumn{2}{c}{\textbf{Minerva}} & 
\multicolumn{2}{c}{\textbf{Olympiad}} & 
\multicolumn{2}{c}{\textbf{Macro Average}} \\
\cmidrule(lr){2-3} \cmidrule(lr){4-5} \cmidrule(lr){6-7} \cmidrule(lr){8-9} \cmidrule(lr){10-11}
& Pass@1 & \# Tokens & Pass@1 & \# Tokens & Pass@1 & \# Tokens & Pass@1 & \# Tokens & Pass@1 & \# Tokens \\
\midrule
\textrm{DeepScaleR-1.5B} & 
38.75 & 8140 & 
85.92 & 3019 & 
27.62 & 4134 & 
51.94 & 5410 & 
51.06 & 5176 \\
\; \textrm{+ GRPO} & 
29.58 & 6122 &
85.28 & 1767 & 
29.64 & 2630 & 
\textbf{49.11} & 3418 & 
\cellann{48.40}{\RedDown{2.66}} & \cellann{3484}{\GreenDown{33\%}} \\
\; \textrm{+ L1-Exact} & 
22.02 & 4138 &
82.57 & 3734 &
28.74 & 3755 &
44.91 & 3987 &
\cellann{44.56}{\RedDown{6.50}} & \cellann{3904}{\GreenDown{24\%}} \\
\;\;\; \textrm{+ L1-Max} &
25.12 & \textbf{2979} &
83.45 & 1912 &
28.35 & \textbf{1701} &
44.76 & \textbf{2334} &
\cellann{45.42}{\RedDown{5.64}} & \cellann{\textbf{2231}}{\GreenDown{57\%}} \\
\; \textrm{+ Efficient-Reasoning} & 
32.71 & 6440 &
82.01 & 1585 &
27.11 & 1852 &
46.95 & 3230 &
\cellann{47.19}{\RedDown{3.87}} & \cellann{3277}{\GreenDown{37\%}} \\
\; \textrm{+ ThinkPrune-Iter2k} & 
\textbf{35.00} & 4598 &
84.22 & 1802 &
29.24 & 1941 &
48.77 & 2935 &
\cellann{49.71}{\RedDown{1.35}} & \cellann{2819}{\GreenDown{45\%}} \\
\; \textbf{+ \texttt{LACONIC} ($\mathbf{B}$=2000)} & 
34.79 & 3976 & 
\textbf{85.61} & \textbf{1556} & 
\textbf{31.76} & 1869 & 
48.96 & 2448 & 
\cellann{\textbf{50.28}}{\RedDown{0.78}} & \cellann{2462}{\GreenDown{52\%}} \\
\midrule
\textrm{DeepSeek-1.5B} & 
30.00 & 15218 &
82.90 & 5340 & 
29.13 & 6783 & 
44.80 & 12417 &
46.71 & 9940 \\
\; \textrm{+ GRPO} &
\textbf{25.83} & 9090 &
81.57 & 2403 & 
27.20 & 2655 & 
43.52 & 5420 &
\cellann{44.53}{\RedDown{2.18}} & \cellann{4892}{\GreenDown{51\%}} \\
\; \textrm{+ Efficient-Reasoning} &
22.08 & 8801 &
81.63 & 1909 &
27.57 & 2018 & 
42.96 & 4799 & 
\cellann{43.56}{\RedDown{3.15}} & \cellann{4382}{\GreenDown{56\%}} \\
\; \textrm{+ ThinkPrune-Iter2k} &
23.33 & 5440 &
81.35 & 1762 &
27.35 & 1837 &
43.39 & \textbf{3021} &
\cellann{43.85}{\RedDown{2.86}} & \cellann{3015}{\GreenDown{70\%}} \\
\; \textbf{+ \texttt{LACONIC} ($\mathbf{B}$=1500)} &
24.17 & \textbf{5140} &
\textbf{81.88} & \textbf{1636} &
\textbf{28.36} & \textbf{1638} &
\textbf{44.11} & 3241 &
\cellann{\textbf{44.63}}{\RedDown{2.08}} & \cellann{\textbf{2914}}{\GreenDown{71\%}} \\
\bottomrule
\end{tabular}
}
\vspace{0.5em}
\label{tab:main-results}
\end{table*}

\section{Experiment}\label{sec:experiment}

\subsection{Experimental Setup}


\textbf{Models and Datasets.} For the training dataset, we use DeepScaleR-Preview-Dataset \citep{deepscaler2025}, a math dataset containing 40.3K rows of question-answer pairs sampled from AIME (prior to 2023), AMC (prior to 2023), Omni-MATH \citep{gao2024omnimathuniversalolympiadlevel}, and STILL \citep{min2024imitateexploreselfimprovereproduction}. 
For base models, we use DeepScaleR-1.5B-Preview \citep{deepscaler2025} (\textbf{DeepScaleR-1.5B} for short) and DeepSeek-R1-Distill-Qwen-1.5B \cite{deepseekai2025deepseekr1} (\textbf{DeepSeek-1.5B} for short).
DeepSeek-1.5B is a 1.5B-parameter model distilled from Qwen2.5-1.5B \citep{qwen2025qwen25technicalreport}, and DeepScaleR-1.5B is a 1.5B-parameter reasoning model further fine-tuned from DeepSeek-1.5B on DeepScaleR-Preview-Dataset.

\textbf{Baselines.} We fine-tune the base models with the following algorithms to serve as baselines and compare with our algorithm \texttt{LACONIC}: (i) GRPO \citep{shao2024deepseekmathpushinglimitsmathematical}: the standard RL-tuning algorithm originally used in the post-training of DeepSeek-1.5B, DeepScaleR-1.5B, and Qwen math models; (ii) L1-Exact and L1-Max \citep{aggarwal2025l1controllinglongreasoning}: heuristic reward design methods that fine-tune models to satisfy target-length constraints; (iii) Efficient-Reasoning \citep{arora2025traininglanguagemodelsreason}: a length-aware reward design controlled by a fixed penalty coefficient; (iv) ThinkPrune-Iter \citep{hou2025thinkprune}: an iterative method to prune thinking lengths.

\textbf{Training.} We train DeepScaleR-1.5B on \texttt{LACONIC} with $B=2000$ for 300 steps, and DeepSeek-1.5B on \texttt{LACONIC} with $B=1500$ for 500 steps. We set the maximum response length to 4K tokens per prompt during training.

\textbf{Evaluation.} We evaluate models on 4 common mathematics benchmarks: AIME2024, MATH \citep{hendrycksmath2021}, Minerva \citep{NEURIPS2022_18abbeef}, and Olympiad-Bench \citep{he-etal-2024-olympiadbench}. To assess the mathematical reasoning ability of the models, we report pass@1, the fraction of questions for which the model's first response matches the correct answer. To quantify the verbosity of the model's output, we report the average response length. We set the maximum response length to 32K tokens during evaluation.

\begin{table*}[t]
\centering
\caption{Evaluation results across out-of-domain (OOD) benchmarks.}
\renewcommand{\arraystretch}{1.2}
\resizebox{\textwidth}{!}{%
\begin{tabular}{lcccccccc}
\toprule
\multirow{2}{*}{\textbf{Model}} &
\multicolumn{2}{c}{\textbf{GPQA}} & 
\multicolumn{2}{c}{\textbf{LSAT}} & 
\multicolumn{2}{c}{\textbf{MMLU}} & 
\multicolumn{2}{c}{\textbf{Macro Average}} \\
\cmidrule(lr){2-3} \cmidrule(lr){4-5} \cmidrule(lr){6-7} \cmidrule(lr){8-9}
& Accuracy & \# Tokens & Accuracy & \# Tokens & Accuracy & \# Tokens & Accuracy & \# Tokens\\
\midrule
\textrm{DeepScaleR-1.5B} & 
 & & & & & & & \\
\; \textrm{+ GRPO} & 	
\textbf{28.88} & 2229 &
24.6 & 3426 &
43.30 &	719 & 
\textbf{32.26} & 2125 \\
\; \textrm{+ L1-Exact} & 
22.88 & 1475 &
25.19 & \textbf{1374} &
36.67 & 766 &
\cellann{28.25}{\RedDown{4.01}} & \cellann{1205}{\GreenDown{43\%}} \\
\;\;\; \textrm{+ L1-Max} & 
28.72 & 1773 &
\textbf{26.58} & 2321 &
38.48 & 863 &
\cellann{31.26}{\RedDown{1.00}} & \cellann{1652}{\GreenDown{22\%}} \\
\; \textrm{+ ThinkPrune-Iter2k} & 
27.02 & 1791 &
24.82 & 3244 &
42.42 & 611 &
\cellann{31.42}{\RedDown{0.84}} & \cellann{1882}{\GreenDown{11\%}} \\
\; \textrm{+ \textbf{\texttt{LACONIC}}} &
27.20 & \textbf{1167} &
24.53 & 1923 &
\textbf{43.36} & \textbf{497} &
\cellann{31.69}{\RedDown{0.57}} & \cellann{\textbf{1196}}{\GreenDown{44\%}} \\
\bottomrule
\end{tabular}
}
\vspace{0.5em}
\label{tab:OOD-results}
\end{table*}

\begin{figure*}[t]
  \centering
  \begin{subfigure}[t]{0.495\textwidth}
    \centering
    \includegraphics[width=\linewidth]{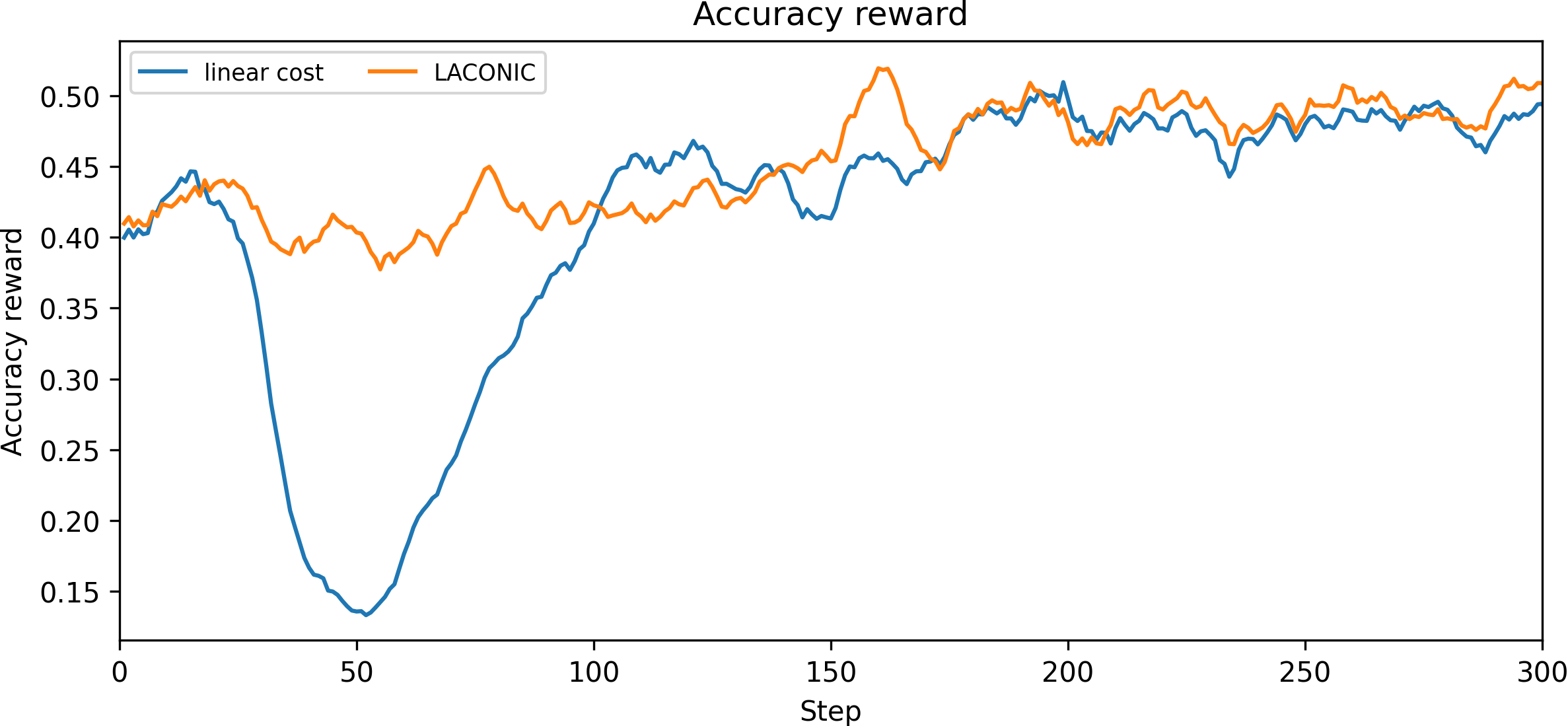}\vspace{-0.1cm}
    \caption{Accuracy reward over training steps}
    \label{fig:reward_accuracy_cost}
  \end{subfigure}\hfill
  \begin{subfigure}[t]{0.495\textwidth}
    \centering
    \includegraphics[width=\linewidth]{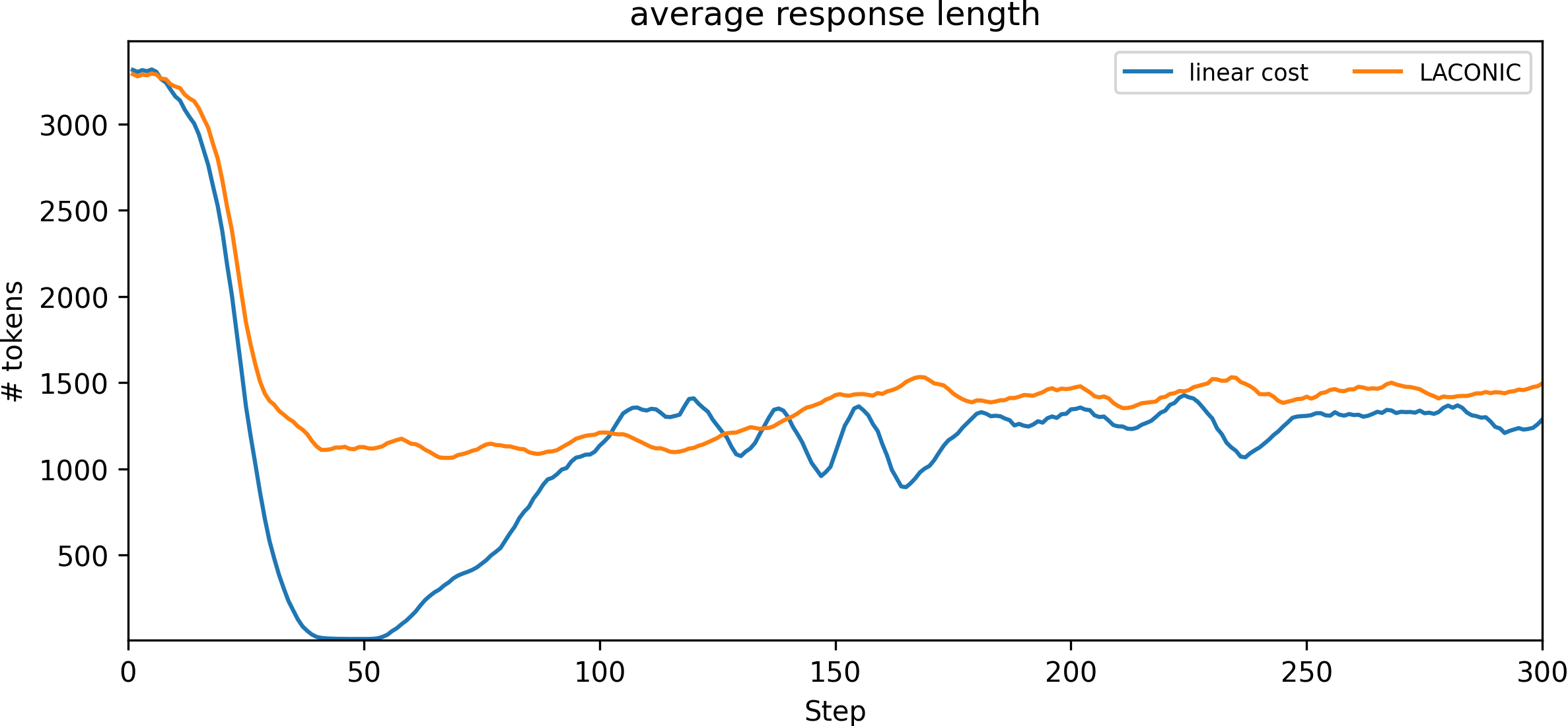}\vspace{-0.1cm}
    \caption{Average response length over training steps}
    \label{fig:mean_response_length_cost}
  \end{subfigure}
  \vspace{-0.1cm}
  \caption{
    Ablation of the cost functions on DeepScaleR-1.5B with token budget $B=1500$. We plot (a) accuracy reward and (b) average response length over training steps. For the linear cost, the Langrangian reward used in primal updates is computed by $r(q,o) - \lambda\,\tilde{c}(q,o)$, where $\tilde{c}(q,o) = (L(o)-B)/B$. All other experiment setups and hyperparameters are identical.
  }
  \label{fig:ablation_cost_main}
\end{figure*}

\subsection{Main Results}
In this section, we first present in \cref{tab:main-results} the main results of all baselines and \texttt{LACONIC} on the 4 mathematics benchmarks. Then we present in \cref{tab:OOD-results} the results on out-of-domain benchmarks, GPQA, LSAT, and MMLU, which probe general knowledge and logic reasoning.

\textbf{\texttt{LACONIC} outperforms existing length-control methods and matches vanilla RL-tuning while significantly reducing response lengths.} On DeepScaleR-1.5B, after fine-tuning, \texttt{LACONIC} achieves 50.28 macro-average pass@1 with 2462 tokens, outperforming all baselines. Existing length-control methods either lose substantially more accuracy, or use noticeably more tokens at similar accuracy. Relative to the full-length base model, \texttt{LACONIC} achieves virtually the same pass@1 while using 52\% fewer tokens.

On DeepSeek-1.5B, \texttt{LACONIC} attains the highest macro-average pass@1 and the lowest token count among all baselines. Relative to the full-length base model, \texttt{LACONIC} reduces response length by 71\% with a modest 2.08 pass@1 decrease.



\textbf{\texttt{LACONIC} preserves out-of-domain (OOD) capabilities.} \texttt{LACONIC} preserves GRPO's macro average accuracy while generating 44\% fewer tokens on average.  Compared with both L1 variants and ThinkPrune-Iter2k, \texttt{LACONIC} attains higher macro pass@1 with fewer tokens. The results shows that \texttt{LACONIC} achieves strong OOD task reward preservation with substantially shorter outputs.




\begin{figure*}[t]
  \centering
  \begin{subfigure}[t]{0.33\textwidth}
    \centering
    \includegraphics[width=\linewidth]{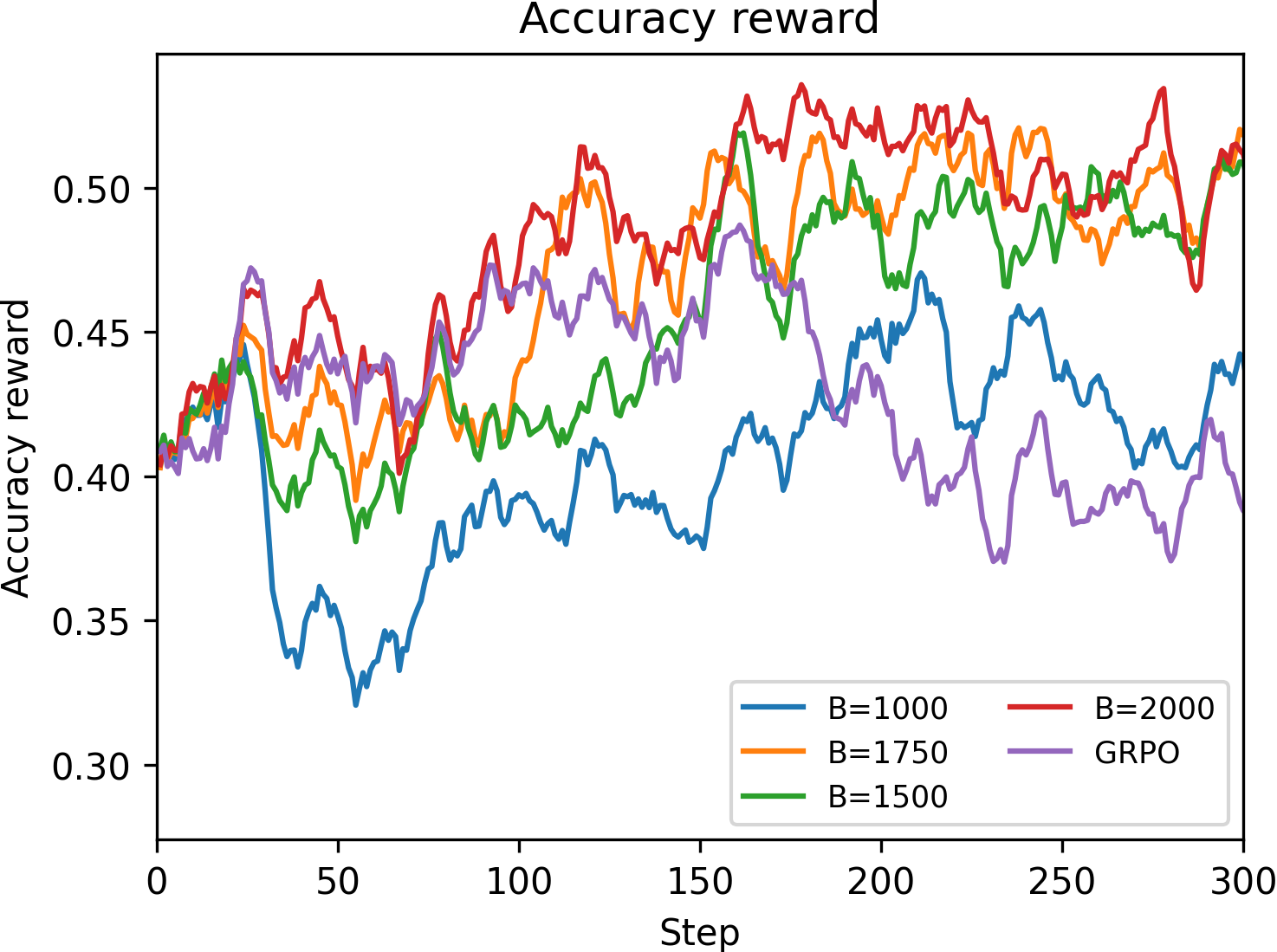}\vspace{-0.1cm}
    \caption{Accuracy reward}
    \label{fig:reward_accuracy_budget}
  \end{subfigure}\hfill
  \begin{subfigure}[t]{0.33\textwidth}
    \centering
    \includegraphics[width=\linewidth]{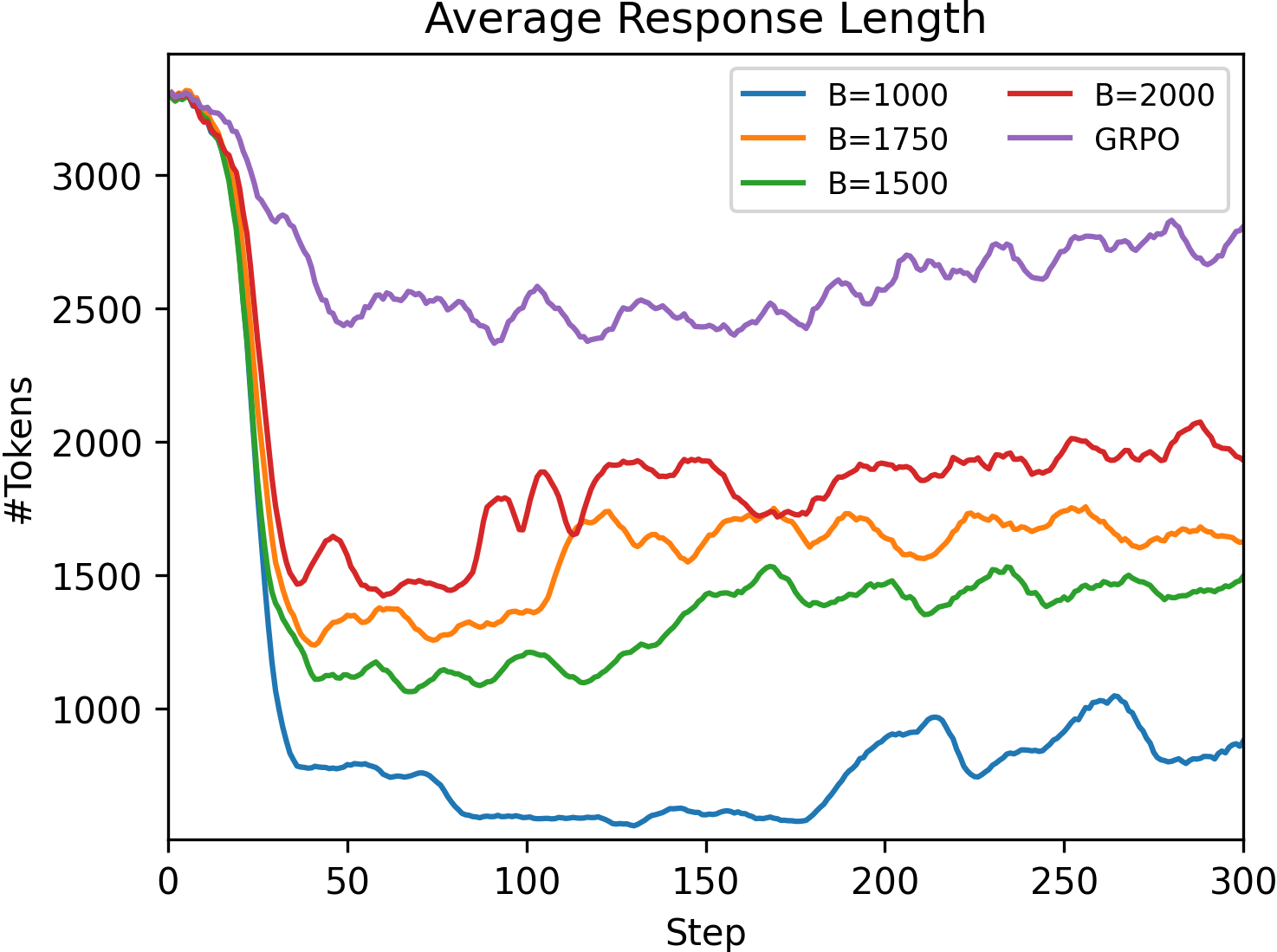}\vspace{-0.1cm}
    \caption{Average response length}
    \label{fig:mean_response_length_budget}
  \end{subfigure}
  \begin{subfigure}[t]{0.33\textwidth}
    \centering
    \includegraphics[width=\linewidth]{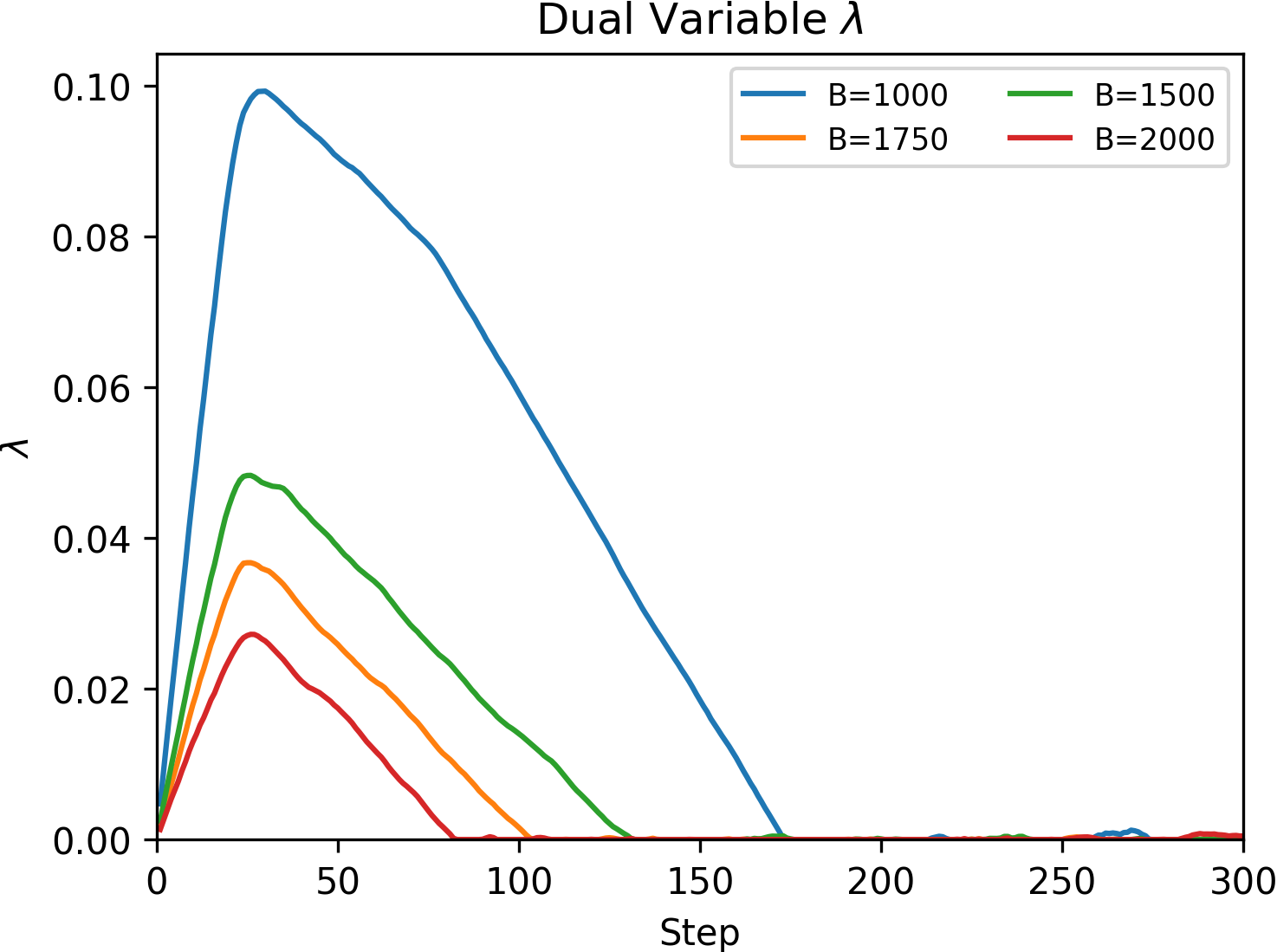}\vspace{-0.1cm}
    \caption{Dual variable $\lambda$}
    \label{fig:lambda_budget}
  \end{subfigure}
  \vspace{-0.1cm}
  \caption{
    Ablation of the token budget $B$ on DeepScaleR-1.5B. We plot (a) accuracy reward; (b) average response length; and (c) dual variable $\lambda$ over training steps with budgets $B\in\{1000,1500,1750,2000\}$. All other setups and hyperparameters are identical.
  }
  \label{fig:ablation_budget_main}
\end{figure*}

\begin{table*}[h]
\centering
\caption{Evaluation results of the token-budget $B$ ablation on DeepScaleR-1.5B across four math benchmarks.}
\renewcommand{\arraystretch}{1.2}
\resizebox{\textwidth}{!}{%
\begin{tabular}{lcccccccccc}
\toprule
\multirow{2}{*}{\textbf{Token Budget}} &
\multicolumn{2}{c}{\textbf{AIME2024}} & 
\multicolumn{2}{c}{\textbf{MATH}} & 
\multicolumn{2}{c}{\textbf{Minerva}} & 
\multicolumn{2}{c}{\textbf{Olympiad}} & 
\multicolumn{2}{c}{\textbf{Macro Average}} \\
\cmidrule(lr){2-3} \cmidrule(lr){4-5} \cmidrule(lr){6-7} \cmidrule(lr){8-9} \cmidrule(lr){10-11}
& Pass@1 & \# Tokens & Pass@1 & \# Tokens & Pass@1 & \# Tokens & Pass@1 & \# Tokens & Pass@1 & \# Tokens \\
\midrule
\textrm{+ \texttt{LACONIC} (2000)} & 
34.79 & 3976 & 
85.61 & 1556 & 
31.76 & 1869 & 
48.96 & 2448 & 
50.28 & 2462 \\
\textrm{+ \texttt{LACONIC} (1750)} &
31.21 & 3803 &
84.82 & 1516 &
31.52 & 2059 &
48.88 & 2261 &
49.11 & 2410 \\
\textrm{+ \texttt{LACONIC} (1500)} & 
29.58 & 3458 &
84.35 & 1353 &
31.48 & 1582 &
47.93 & 2084 &
48.33 & 2119 \\
\textrm{+ \texttt{LACONIC} (1000)} & 
28.12 & 2665 &
83.75 & 1049 &
29.24 & 1189 &
47.07 & 1669 &
47.05 & 1643 \\
\bottomrule
\end{tabular}
}
\vspace{0.5em}
\label{tab:ablation-results}
\end{table*}

\section{Further Analysis}\label{sec:analysis}



In this section, we present additional ablation analysis related to cost functions, hyperparameters, and examine the computational resources required by \texttt{LACONIC}, including runtime, FLOPs, and memory usage.

\subsection{Ablation Analysis on Cost Function}\label{sec:cost_ablation}
As discussed in \cref{sec:formulation}, a primal-dual formulation provides a principled way to enforce a target token budget $B$. However, using a linear cost $\tilde{c}(q,o) = \frac{L(o) - B}{B}$ can induce unstable policy updates in practice. In this section , we isolate this design choice by comparing the linear cost $\tilde{c}$ against the clipped cost $c$ we propose for \texttt{LACONIC}.

We change only the cost used in the GRPO-style primal updates while keeping all other setups identical, and report the ablation results in \cref{fig:ablation_cost_main}. The linear-cost baseline exhibits a sharp length collapse early in training, driving the model to generate extremely short, degenerate outputs. This causes unstable policy updates and a drastic drop in reward. In contrast, \texttt{LACONIC} maintains stable learning dynamics.

After training, \texttt{LACONIC} achieves 48.33 macro pass@1 with 2119 tokens on average, while the linear-cost baseline attains 47.37 macro pass@1 but produces 3281 tokens on average. The detailed results are deferred to \cref{sec:app-cost-ablation}.

\subsection{Ablation Analysis on Budget $B$}
We vary the token budget $B\in \{2000, 1750, 1500, 1000\}$ on DeepScaleR-1.5B while keeping all other settings and hyperparameters (including the dual step size) fixed, and train for 300 steps. \Cref{fig:ablation_budget_main} shows the training dynamics of (a) accuracy reward; (b) average response length; and (c) dual variable $\lambda$. We evaluate the step-300 checkpoints on the four mathematics benchmarks. \Cref{tab:ablation-results} reports pass@1 and average response lengths.

\textbf{\texttt{LACONIC} provides reliable, hyperparameter-tuning-free length control.} Across token budgets, training rapidly drives the average response length under the budget and maintains it near the budget once stabilized. Even under tight constraints on a backbone that naturally produces long responses, \texttt{LACONIC} keeps the average length near the budget. In practice, $B$ acts as a single knob and no re-tuning of other hyperparameters is required to achieve effective length control. 


\textbf{\texttt{LACONIC} achieves better or matching reward with less tokens than existing length control methods across a wide range of token budgets.} We compare the evaluation results of \texttt{LACONIC} with budgets $B$ from 1K to 2K in \cref{tab:ablation-results} to the full-length base model and baselines in \cref{tab:main-results}. Across a wide range of budgets, \texttt{LACONIC} consistently outperforms or matches all baselines while using substantially fewer tokens. This shows the effectiveness of \texttt{LACONIC} to preserve reward under constrained token budgets.




\subsection{Ablation Analysis on Dual Step Size $\eta$}\label{sec:ablation_eta}
We vary the step size for dual updates $\eta \in \{0.001,0.002,0.01\}$ with all other settings fixed and train DeepScaleR-1.5B for with $B=1500$. For $\eta=0.01$, we also set a low $\lambda$-ceiling $\Lambda = 0.1$. \Cref{fig:ablation_eta_main} plots the training dynamics of (a) accuracy reward; (b) average response length; and (c) dual variable $\lambda$. We evaluate the step-300 checkpoints on the 4 mathematics benchmarks, and report the pass@1 and average response lengths in \cref{sec:ablation-results-eta}. 

\textbf{\texttt{LACONIC} is insensitive to the dual step size $\eta$.} Across different dual step sizes, \texttt{LACONIC} delivers consistent length control. With $\eta=0.002$ versus 0.01, the training dynamics are similar after stabilization. All trained checkpoints reach comparable final reward. This shows that \texttt{LACONIC} is robust to an order-of-magnitude change in the step size $\eta$.

\begin{figure}[b]
    \vspace{-0.4cm}
    \centering
    \includegraphics[width=\linewidth]{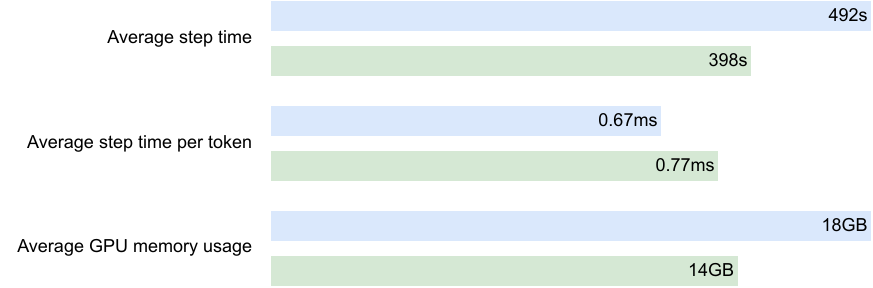}
    \vspace{-0.4cm}
    \caption{Average computational resource usage of \texttt{LACONIC} (green) and GRPO (blue).}
    \label{fig:resource}
    \vspace{-0.3cm}
\end{figure}



\begin{figure*}[t]
  \centering
  \begin{subfigure}[t]{0.33\textwidth}
    \centering
    \includegraphics[width=\linewidth]{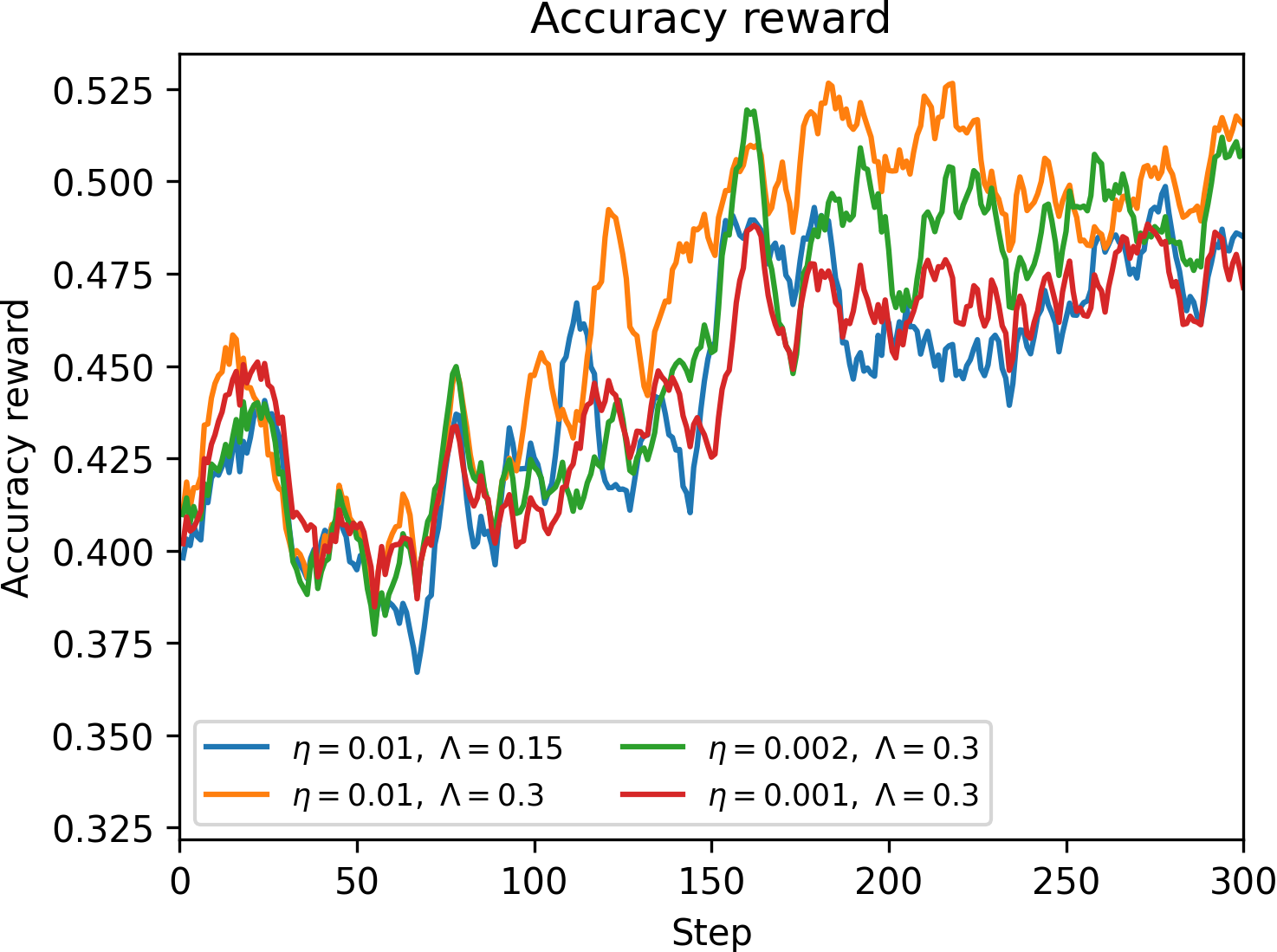}\vspace{-0.1cm}
    \caption{Accuracy reward}
    \label{fig:reward_accuracy_eta}
  \end{subfigure}\hfill
  \begin{subfigure}[t]{0.33\textwidth}
    \centering
    \includegraphics[width=\linewidth]{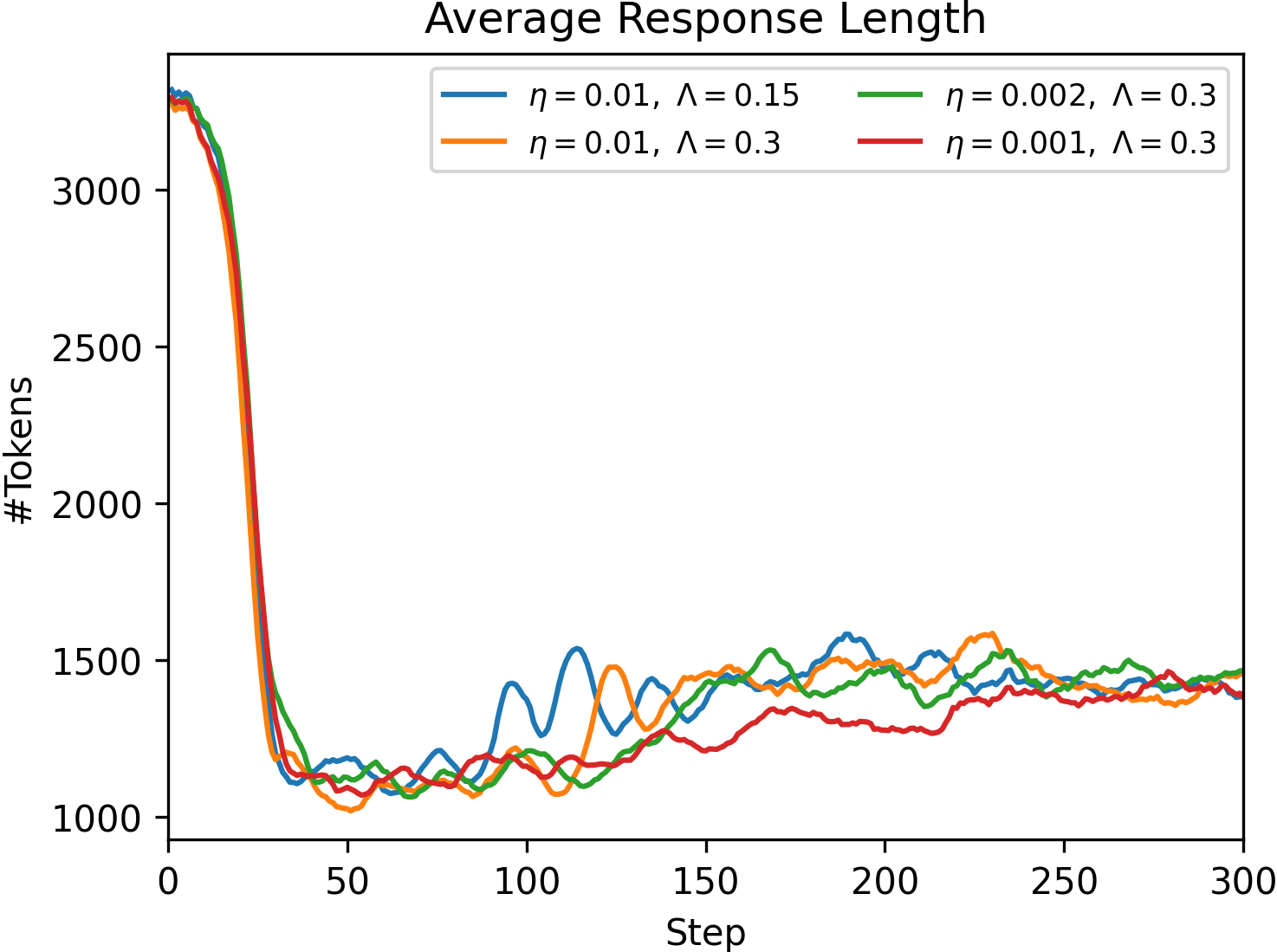}\vspace{-0.1cm}
    \caption{Average response length}
    \label{fig:mean_response_length_eta}
  \end{subfigure}
  \begin{subfigure}[t]{0.33\textwidth}
    \centering
    \includegraphics[width=\linewidth]{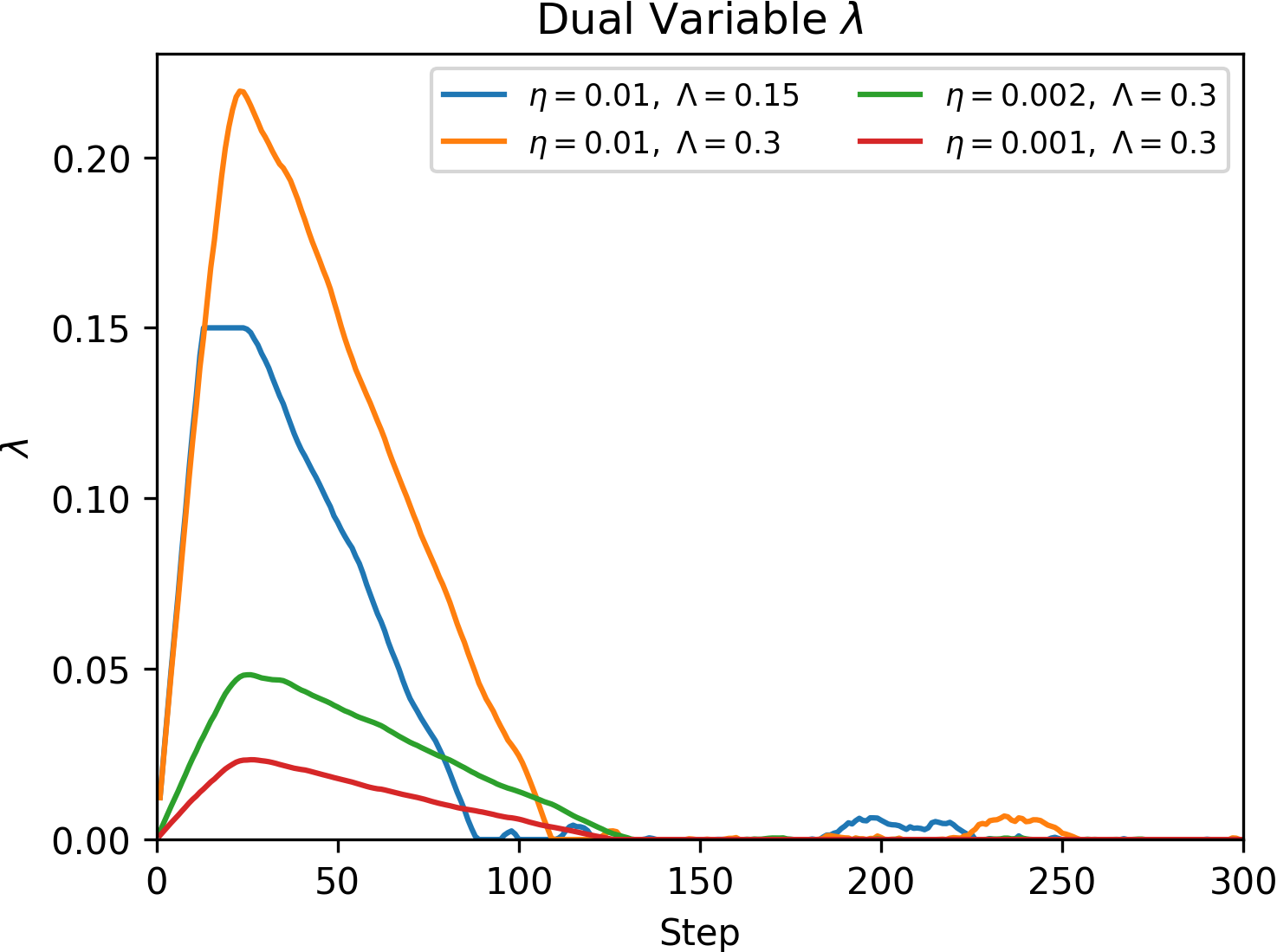}\vspace{-0.1cm}
    \caption{Dual variable $\lambda$}
    \label{fig:lambda_eta}
  \end{subfigure}
  \vspace{-0.1cm}
  \caption{
    Ablation of the dual step size $\eta$ on DeepScaleR-1.5B with $B=1500$. We plot (a) accuracy reward; (b) average response length; and (c) dual variable $\lambda$ over training steps with step size $\eta\in\{0.001,0.002,0.01\}$. We set $\Lambda=0.3$ for all experiments except for the run with $\eta=0.01$ and $\Lambda=0.15$. All other setups and hyperparameters are identical.
  }
  \label{fig:ablation_eta_main}
\end{figure*}





\subsection{Computational Resource Analysis}


We report in \cref{fig:resource} wall-clock step time, step time per token, and NVML GPU memory, averaged over training steps. \texttt{LACONIC} is end-to-end cheaper than vanilla RL-tuning. Our method is 19\% faster and uses 22\% less GPU memory. Per-token cost is nearly unchanged, with a small bookkeeping overhead for the length cost and dual update. Overall, \texttt{LACONIC} adds negligible kernel-level overhead and reduces runtime and memory by generating fewer tokens.

\section{Related Work}

\textbf{RL fine-tuning.}
Reinforcement learning has become a crucial component of LLM post-training, particularly for enhancing large-scale reasoning and aligning model behavior with human preferences or task-specific objectives. Starting from the policy gradient algorithm \citep{Sutton1998},following works \citep{williams1992simple,schulman2017proximalpolicyoptimizationalgorithms} addressed the instability of early methods. Several GRPO-based extensions have been introduced to address specific challenges in LLM training. SRPO \citep{zhang2025srpo} addresses the problem of ineffective samples through history resampling. DAPO \citep{yu2025dapo} introduces dynamic sampling to better handle complex reasoning tasks such as chain-of-thought (CoT) generation. Other variants include VAPO \citep{yue2025vapo}, which adapts advantage estimation to better capture variance across different reasoning depths; GSPO \citep{zheng2025group}, which emphasizes group-level structure in sampling; GFPO \citep{shrivastava2025samplethinklessgroup}, which focuses on sample efficiency in long-horizon settings; GMPO \citep{zhao2025geometricmean}, which explores geometric averaging of policy gradients for improved robustness; and Dr.GRPO \citep{liu2025understandingr1zeroliketrainingcritical}, which proposes updated formulae for computing advantages and the objective for GRPO. Collectively, these methods demonstrate the growing sophistication of RL fine-tuning techniques and the community’s effort to make them more scalable, stable, and effective for large-scale LLM alignment.

\textbf{Length-Aware LLMs.} Recent work has investigated various approaches for making large language models (LLMs) aware of output length, including modifications to positional encoding, prompt engineering techniques, and post-hoc truncation methods \citep{li2025search, wu2025agentic, li2025webthinker}. A common strategy involves incorporating length preferences into reinforcement learning (RL) fine-tuning through manually designed reward functions that penalize or incentivize certain output lengths \citep{aggarwal2025l1controllinglongreasoning, cheng2025incentivizingdualprocessthinking, huang2025hapotraininglanguagemodels, yuan2025efficientrltrainingreasoning,arora2025traininglanguagemodelsreason,hou2025thinkprune}. These methods typically rely on fixed heuristics or penalty terms that remain constant throughout training, and thus optimize a surrogate objective that may be misaligned with the true downstream task reward. This misalignment can lead to suboptimal performance and often requires extensive hyperparameter tuning to balance length control with task-specific quality. Our work differs in that it aims to align length control with task reward in a more adaptive and data-driven manner, avoiding the limitations of fixed shaping objectives.

\textbf{Constrained RL.} Constrained RL is commonly used to formulate and enforce constraints in the environment. Many previous studies \citep{achiam2017constrainedpolicyoptimization,tessler2018rewardconstrainedpolicyoptimization,stooke2020responsivesafetyreinforcementlearning} have proposed policy gradient methods for constrained RL. Recent work including this paper has been introducing the constrained RL framework into RL fine-tuning for LLMs. \cite{tzannetos2025curriculumdesigntrajectoryconstrainedagent} propose a curriculum strategy to compress inference time of LLMs. \cite{zhang2025alignmentlargelanguagemodels}  deploy the primal-dual approach to control the divergence of policy updates.


\section{Conclusion}
We present \texttt{LACONIC}, a primal-dual method for length-aware RL fine-tuning that integrates into standard GRPO with minimal changes. It enforces a user-specified token budget via a clipped cost and an adaptive dual variable for excess length, yielding concise generations while preserving accuracy. We also provide a theoretical analysis with convergence and near-optimality guarantees. Across math and out-of-domain benchmarks, \texttt{LACONIC} consistently reduces output length while preserving pass@1. Budget-only ablations show precise, stable controllability that lengths track the target budget without retuning hyperparameters. Overall, \texttt{LACONIC} makes length control a simple, reliable component of RL-based LLM fine-tuning.


\section{Future Work}

While \texttt{LACONIC} is effective and lightweight, it has several limitations. It currently enforces a global average token budget, which may not capture prompt-specific or context-dependent needs. Our experiments are limited to math reasoning, and future work could validate generality on dialogue, summarization, or code. Finally, \texttt{LACONIC} handles a single constraint, but the framework naturally extends to multi-constraint settings such as latency or safety.



\bibliography{custom}
\bibliographystyle{iclr2026_conference}




\newpage
\appendix
\onecolumn
\crefalias{section}{appendix}
\crefalias{subsection}{appendix}


\section{Standard Primal-Dual Method and Linear Cost Functions}\label{app:ablation_cost}
\subsection{Standard Primal-Dual Updates}
In this section, we derive the standard primal-dual updates.

Recall the constrained optimization problem of length-aware LLMs formulated in \cref{eq:CMDP}:
$$\max_\theta\EE_{q\sim P(Q), o\sim\pi_\theta(\cdot|q)}[r(q,o)], \; \textrm{s.t. }\; \EE_{q\sim P(Q), o\sim\pi_\theta(\cdot|q)}[L(o)] \le B.$$
The corresponding Lagrangian is 
\begin{equation}
    \cL(\theta,\lambda) = \EE_{q\sim P(Q), o\sim\pi_\theta(\cdot|q)}[r(q,o)] - \lambda\, (\frac{\EE_{q\sim P(Q), o\sim\pi_\theta(\cdot|q)}[L(o)]}{B}-1), \quad \lambda\ge 0,
\end{equation}
and the problem is commonly approached by solving the saddle point problem: 
\begin{equation}\label{eq:saddle-point}
    \max_\theta\min_{\lambda\ge 0} \cL(\theta,\lambda),
\end{equation}
where $\theta$ is the primal variable (in our case, the policy model) and $\lambda$ is the dual variable. We assume that \cref{eq:saddle-point} is feasible, and denote the optimal feasible solution by $(\pi^*, \lambda^*)$. The standard primal-dual approach solves $\pi^*$ and $\lambda^*$ iteratively with partial derivatives. In the primal step, the dual variable $\lambda_t$ is fixed, and we optimize $\cL(\theta, \lambda_t)$ over $\theta$:
\begin{equation*}
    \theta_{t+1}\in\argmax_\theta \EE_{q\sim P(Q), o\sim\pi_\theta(\cdot|q)}[r(q,o)] - \lambda_t\, (\frac{\EE_{q\sim P(Q), o\sim\pi_\theta(\cdot|q)}[L(o)]}{B}-1).
\end{equation*}
By linearity of expectations, we have
\begin{equation*}
    \theta_{t+1} \in \argmax_\theta \EE_{q\sim P(Q), o\sim\pi_\theta(\cdot|q)}[r(q,o) - \lambda_t \frac{L(o) - B}{B}].
\end{equation*}
Define the linear cost function 
\begin{equation}\label{eq:linear-cost}
    \tilde{c}(q,o) = \frac{L(o) - B}{B},
\end{equation}
and thus the primal update takes the form
\begin{equation}\label{eq:ideal_standard_primal}
    \theta_{t+1} \in \argmax_\theta \EE_{q\sim P(Q), o\sim\pi_\theta(\cdot|q)}[r(q,o) - \lambda_t \tilde{c}(q,o)].
\end{equation}
In the dual update, we decrease $\cL(\theta_{t+1},\lambda)$ by taking a gradient-descent step for $\lambda$:
\begin{equation}
    \lambda_{t+1} \leftarrow \max\{\lambda_t - \eta \frac{\partial \cL(\theta_{t+1},\lambda)}{\partial \lambda},0\} = \max\{\lambda + \eta (\frac{\EE_{q\sim P(Q), o\sim\pi_{\theta_{t+1}}(\cdot|q)}[L(o)]}{B}-1),0\}.
\end{equation}
Thus far, we have derived the idealized standard primal dual updates in \cref{eq:ideal_primal,eq:ideal_dual}. Empirically we estimate the expectation $\EE_{q\sim P(Q), o\sim\pi_{\theta_t}(\cdot|q)}[L(o)]$ with the minibatch mean $\bar{L}$ generated by the policy model $\pi_\theta$, and with an extra $\lambda$-ceiling $\Lambda$, we derive exactly \cref{eq:dual-update}:
$$\lambda \leftarrow \textrm{clip}(\lambda + \eta(\frac{\bar{L}}{B} - 1), 0, \Lambda).$$

For the primal update in \cref{eq:ideal_standard_primal}, we treat it as a RL-tuning step where the task reward is substituted with a Langrangian reward $\hat{\ell}_{\lambda_t}(q,o)\defeq r(q,o) - \lambda_t \tilde{c}(q,o)$.
We refer to the aforementioned method as the standard primal-dual updates (with linear costs). Note that \texttt{LACONIC} shares the same dual updates as the standard primal-dual method, while changing the cost function used in primal updates.

\subsection{Additional Cost Function Ablations}\label{sec:app-cost-ablation}
As discussed in \cref{sec:formulation,sec:cost_ablation}, using the linear cost function $\tilde{c}(q,o) = \frac{L(o) - B}{B}$ inside the primal (model) update is problematic in practice. Whenever $\lambda>0$, the actor model is incentivized to shorten responses on all samples, and eventually collapses the policy to extremely short outputs. This leads to highly unstable training dynamics, which is undesirable in practice. In \cref{sec:cost_ablation}, we trained DeepScaleR-1.5B using two different methods. We present in \cref{tab:eval_deepscaler_cost_ablation} the detailed evaluation results of the step-300 checkpoints of DeepScaleR-1.5B on these two methods with a token budget $B=1500$. 

\begin{table*}[h]
\centering
\caption{Evaluation results of the cost function ablation on DeepScaleR-1.5B across four math benchmarks.}
\renewcommand{\arraystretch}{1.2}
\resizebox{\textwidth}{!}{%
\begin{tabular}{lcccccccccc}
\toprule
\multirow{2}{*}{\textbf{Methods}} &
\multicolumn{2}{c}{\textbf{AIME2024}} & 
\multicolumn{2}{c}{\textbf{MATH}} & 
\multicolumn{2}{c}{\textbf{Minerva}} & 
\multicolumn{2}{c}{\textbf{Olympiad}} & 
\multicolumn{2}{c}{\textbf{Macro Average}} \\
\cmidrule(lr){2-3} \cmidrule(lr){4-5} \cmidrule(lr){6-7} \cmidrule(lr){8-9} \cmidrule(lr){10-11}
& Pass@1 & \# Tokens & Pass@1 & \# Tokens & Pass@1 & \# Tokens & Pass@1 & \# Tokens & Pass@1 & \# Tokens \\
\midrule
\textrm{DeepScaleR-1.5B} & 
38.75 & 8140 & 
85.92 & 3019 & 
27.62 & 4134 & 
51.94 & 5410 & 
51.06 & 5176 \\
\; \textrm{+ GRPO} & 
29.58 & 6122 &
\textbf{85.28} & 1767 & 
29.64 & 2630 & 
\textbf{49.11} & 3418 & 
\cellann{\textbf{48.40}}{\RedDown{2.66}} & \cellann{3484}{\GreenDown{33\%}} \\
\; \textrm{+ Linear cost (1500)} & 
29.17 & 5636 &
83.00 & 1593 & 
29.14 & 2638 & 
48.18 & 3256 & 
\cellann{47.37}{\RedDown{3.69}} & \cellann{3281}{\GreenDown{37\%}} \\
\; \textbf{+ \texttt{LACONIC} (1500)} & 
\textbf{29.58} & \textbf{3458} &
84.35 & \textbf{1353} &
\textbf{31.48} & \textbf{1582} &
47.93 & \textbf{2084} &
\cellann{48.33}{\RedDown{2.73}} & \cellann{\textbf{2119}}{\GreenDown{59\%}} \\
\bottomrule
\end{tabular}
}
\vspace{0.5em}
\label{tab:eval_deepscaler_cost_ablation}
\end{table*}

The standard primal-dual approach failed to effectively shorten response length on DeepScaleR-1.5B compared to \texttt{LACONIC} or even vanilla GRPO with a restrictive 4K response length cap.

We then include additional ablations of cost functions on another base model, Qwen2.5-Math-1.5B-Instruct \citep{qwen2.5} (\textbf{Qwen-Math-1.5B} for short). Qwen-Math-1.5B is a 1.5B-parameter instruction-tuned math model. In \cref{fig:qwen_math_ablation_cost}, we plot the training dynamics, and in \cref{tab:qwen-math-ablation-cost-results}, we present the evaluation results of the step-350 checkpoints. The experiment shows that with linear cost function $\frac{L(o)-B}{B}$, the average accuracy plummets from 40\% to 8\%, and the mean response lengths fall below 10 tokens, indicating that the model is no longer producing meaningful responses during training. Although as shown in \cref{tab:qwen-math-ablation-cost-results}, the primal-dual framework with linear cost function $\tilde{c}$ recovers the model's performance after the model is stabilized, the unstable update steps can introduce great risk and instability to LLM fine-tuning.

\begin{table*}[t]
\centering
\caption{Evaluation results of the cost functions on Qwen-Math-1.5B across four math benchmarks.}
\renewcommand{\arraystretch}{1.2}
\resizebox{\textwidth}{!}{%
\begin{tabular}{lcccccccccc}
\toprule
\multirow{2}{*}{\textbf{Model}} &
\multicolumn{2}{c}{\textbf{AIME2024}} & 
\multicolumn{2}{c}{\textbf{MATH}} & 
\multicolumn{2}{c}{\textbf{Minerva}} & 
\multicolumn{2}{c}{\textbf{Olympiad}} & 
\multicolumn{2}{c}{\textbf{Macro Average}} \\
\cmidrule(lr){2-3} \cmidrule(lr){4-5} \cmidrule(lr){6-7} \cmidrule(lr){8-9} \cmidrule(lr){10-11}
& Pass@1 & \# Tokens & Pass@1 & \# Tokens & Pass@1 & \# Tokens & Pass@1 & \# Tokens & Pass@1 & \# Tokens \\
\midrule
\textrm{Qwen-Math-1.5B} & 
11.46 & 952 &
\textbf{74.86} & 570 & 
25.14 & 656 & 
\textbf{39.84} & 808 &
\textbf{37.89} & 747 \\
\; \textrm{+ Linear cost (550)} &
11.35 & \textbf{665} &
73.42 & 479 &
\textbf{25.53} & 531 &
38.42 & 627 &
\cellann{37.18}{\RedDown{0.71}} & \cellann{576}{\GreenDown{23\%}} \\
\; \textbf{+ \texttt{LACONIC} (550)} &
\textbf{11.46} & 674 &
73.85 & \textbf{464} &
25.25 & \textbf{524} &
38.39 & \textbf{603} &
\cellann{37.24}{\RedDown{0.65}} & \cellann{\textbf{566}}{\GreenDown{24\%}} \\
\bottomrule
\end{tabular}
}
\vspace{0.5em}
\label{tab:qwen-math-ablation-cost-results}
\end{table*}

\begin{figure}[t]
  \centering
  \begin{subfigure}[t]{0.495\textwidth}
    \centering
    \includegraphics[width=\linewidth]{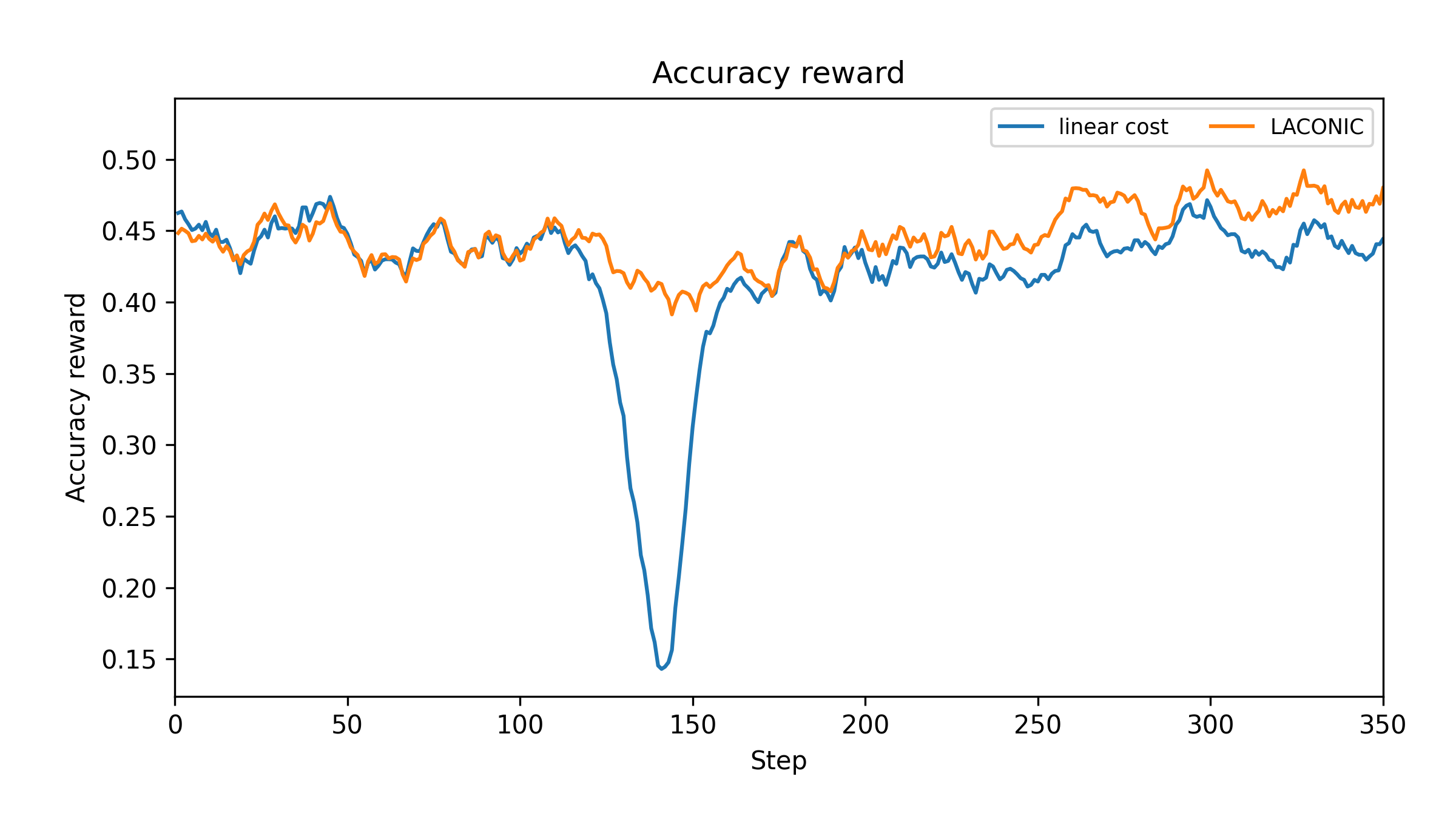}\vspace{-0.3cm}
    \caption{Accuracy reward over training steps}
    \label{fig:reward_accuracy_cost_qwen}
  \end{subfigure}\hfill
  \begin{subfigure}[t]{0.495\textwidth}
    \centering
    \includegraphics[width=\linewidth]{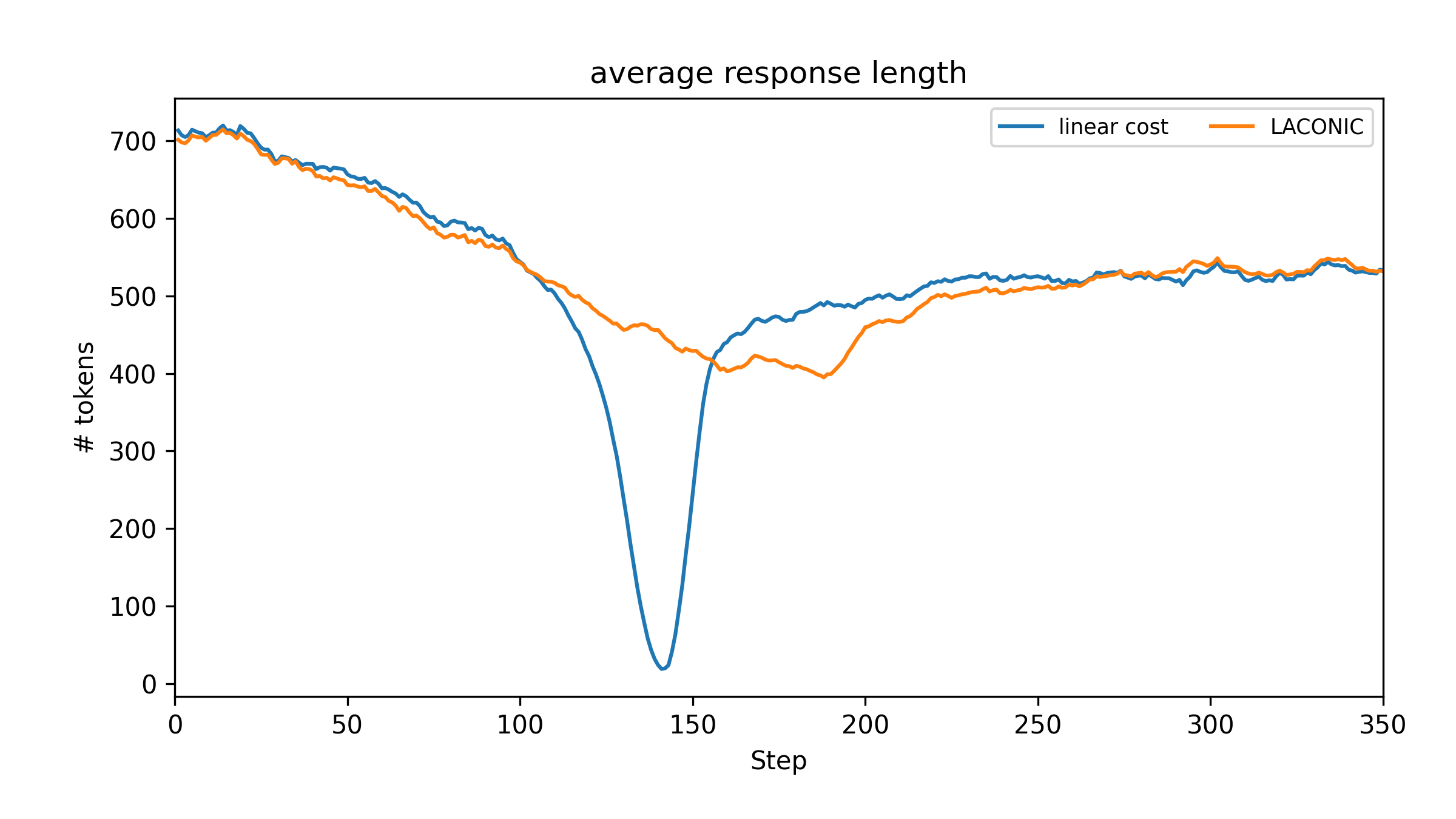}\vspace{-0.3cm}
    \caption{Average response length over training steps}
    \label{fig:mean_response_length_cost_qwen}
  \end{subfigure}
  \vspace{-0.3cm}
  \begin{subfigure}[t]{\textwidth}
    \centering
    \includegraphics[width=0.495\linewidth]{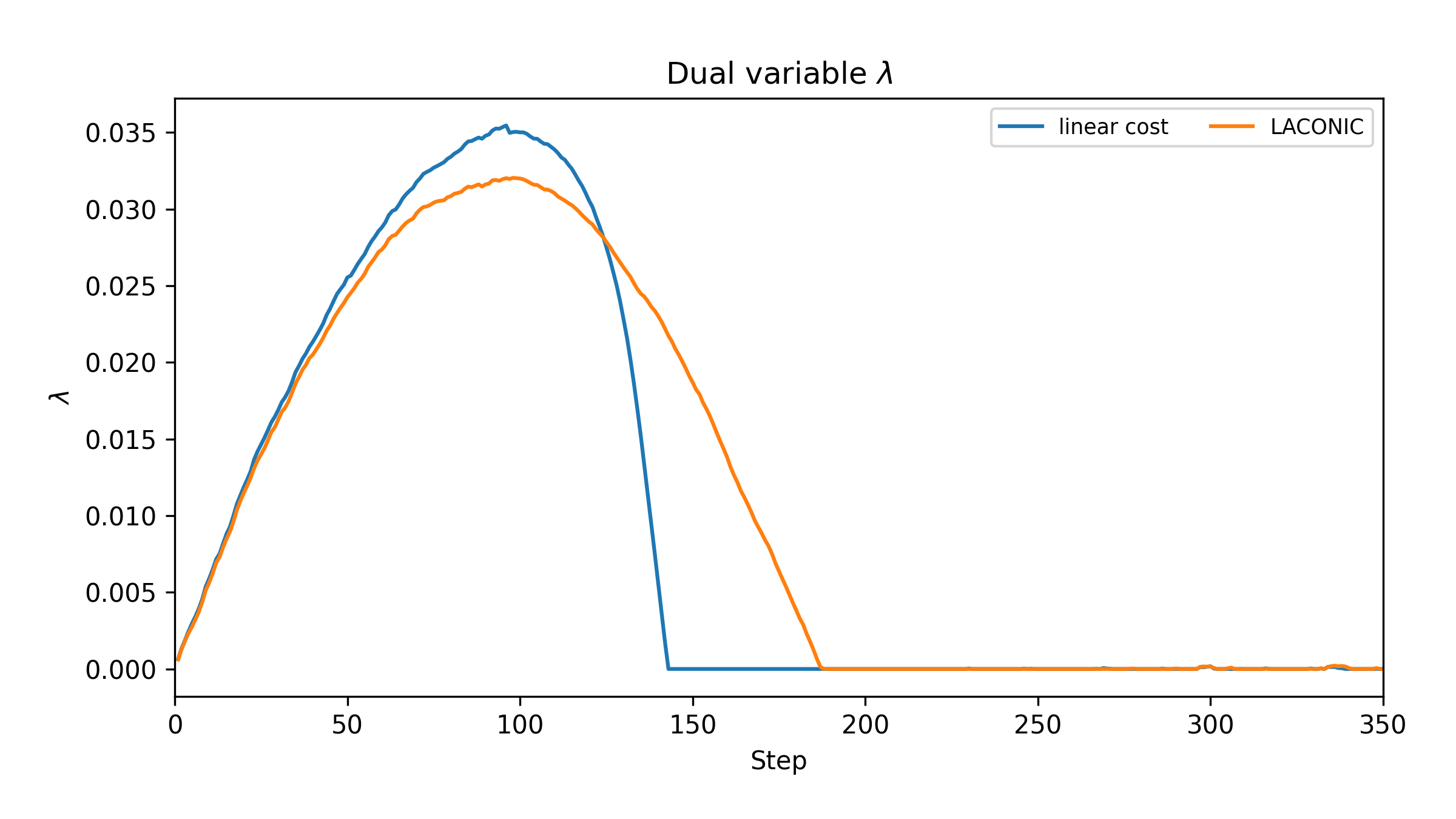}\vspace{-0.3cm}
    \caption{Dual variable \(\lambda\) over training steps}
    \label{fig:lambda_len_cost}
  \end{subfigure}

  \caption{
    Ablation of the cost functions on Qwen-Math-1.5B with a token budget $B=550$:
    (a) reward accuracy,
    (b) average response length,
    (c) dual variable $\lambda$,
  }
  \label{fig:qwen_math_ablation_cost}
\end{figure}

\section{Additional Experiments}
\subsection{Ablation Results on Step Sizes}\label{sec:ablation-results-eta}
In \cref{sec:ablation_eta}, we introduced the ablation experiments on DeepScaleR-1.5B by varying step sizes and $\lambda$-ceiling, and presented training dynamics. In this section, we show the evaluation results of the step-300 checkpoints of all runs.

\begin{table*}[h]
\centering
\caption{Evaluation results of the token-budget $B$ ablation across four math benchmarks.}
\renewcommand{\arraystretch}{1.2}
\resizebox{\textwidth}{!}{%
\begin{tabular}{lcccccccccc}
\toprule
\multirow{2}{*}{\textbf{Step size $\eta$}} &
\multicolumn{2}{c}{\textbf{AIME2024}} & 
\multicolumn{2}{c}{\textbf{MATH}} & 
\multicolumn{2}{c}{\textbf{Minerva}} & 
\multicolumn{2}{c}{\textbf{Olympiad}} & 
\multicolumn{2}{c}{\textbf{Macro Average}} \\
\cmidrule(lr){2-3} \cmidrule(lr){4-5} \cmidrule(lr){6-7} \cmidrule(lr){8-9} \cmidrule(lr){10-11}
& Pass@1 & \# Tokens & Pass@1 & \# Tokens & Pass@1 & \# Tokens & Pass@1 & \# Tokens & Pass@1 & \# Tokens \\
\midrule
\textrm{$\eta=0.001, \Lambda=0.3$} & 
27.50 & 3231 &
84.25 & 1356 &
30.92 & 1543 &
47.87 & 2090 &
47.64 & 2055 \\
\textrm{$\eta=0.002, \Lambda=0.3$} &
29.58 & 3458 &
84.35 & 1353 &
31.48 & 1582 &
47.93 & 2084 &
48.33 & 2119 \\
\textrm{$\eta=0.01, \Lambda=0.3$} & 
29.79 & 3429 &
84.30 & 1365 &
31.42 & 1602 &
47.19 & 2097 &
48.18 & 2123 \\
\textrm{$\eta=0.01, \Lambda=0.15$} & 
29.79 & 3477 &
84.20 & 1319 &
31.15 & 1562 &
47.83 & 2067 &
48.24 & 2106 \\
\bottomrule
\end{tabular}
}
\vspace{0.5em}
\label{tab:ablation-results-eta}
\end{table*}

\subsection{Experiments on Qwen3-8B}
To test the effectiveness of our method on large models, we train Qwen3-8B on \texttt{LACONIC} with a token budget $B=1000$ and 4K response length cap for 150 steps, and we present the evaluation results in \cref{tab:eval_8B}.

The experimental results show that \texttt{LACONIC} reduces response length by 46\% while incurring a modest 2.77\% change in pass@1 compared with vanilla GRPO fine-tuning. Moreover, \texttt{LACONIC} outperforms both L1 variants while using fewer tokens.

\begin{table*}[h]
\centering
\caption{Evaluation results of length-control on Qwen3-8B across four math benchmarks.}
\renewcommand{\arraystretch}{1.2}
\resizebox{\textwidth}{!}{%
\begin{tabular}{lcccccccccc}
\toprule
\multirow{2}{*}{\textbf{Methods}} &
\multicolumn{2}{c}{\textbf{AIME2024}} & 
\multicolumn{2}{c}{\textbf{MATH}} & 
\multicolumn{2}{c}{\textbf{Minerva}} & 
\multicolumn{2}{c}{\textbf{Olympiad}} & 
\multicolumn{2}{c}{\textbf{Macro Average}} \\
\cmidrule(lr){2-3} \cmidrule(lr){4-5} \cmidrule(lr){6-7} \cmidrule(lr){8-9} \cmidrule(lr){10-11}
& Pass@1 & \# Tokens & Pass@1 & \# Tokens & Pass@1 & \# Tokens & Pass@1 & \# Tokens & Pass@1 & \# Tokens \\
\midrule
\textrm{Qwen3-8B} & 
 & & & & & & & & & \\
\; \textrm{+ GRPO} & 
37.50 & 2918 &
87.95 & 1532 & 
40.14 & 1549 & 
54.40 & 2167 & 
55.00 & 2042 \\
\; \textrm{+ L1-Exact} & 
27.92 & 1637 &
86.37 & 987 & 
37.50 & 800 & 
48.54 & 1254 & 
\cellann{50.08}{\RedDown{4.92}} & \cellann{1170}{\GreenDown{43\%}} \\
\;\;\; \textrm{+ L1-Max} & 
30.33 & 1936 &
87.30 & 865 & 
37.59 & 759 & 
51.11 & 1209 & 
\cellann{51.58}{\RedDown{3.42}} & \cellann{1192}{\GreenDown{42\%}} \\
\; \textbf{+ \texttt{LACONIC} (1000)} & 
30.42 & 1643 &
87.01 & 792 & 
38.46 & 783 & 
53.05 & 1166 & 
\cellann{52.23}{\RedDown{2.77}} & \cellann{1096}{\GreenDown{46\%}} \\
\bottomrule
\end{tabular}
}
\vspace{0.5em}
\label{tab:eval_8B}
\end{table*}

\subsection{Ablation Analysis on Fixed Dual Variables}
To isolate the effect of online dual adaptation, we replace \texttt{LACONIC}'s dual update with a fixed $\lambda$ throughout training and sweep $\lambda$ from $\{0.001, 0.01, 0.05\}$. \Cref{tab:ablation-results-fixed-lambda} reports pass@1 and average output length under identical training settings and the same token budget $B=1000$. The sweep reveals strong sensitivity to \(\lambda\): small \(\lambda\) behaves similarly to the base model, yielding limited length reduction, whereas large \(\lambda\) over-emphasizes the length penalty, producing short outputs and clear accuracy degradation. In contrast, \texttt{LACONIC}'s adaptive \(\lambda\) achieves better pass@1 at substantially lower length, without per-task \(\lambda\) tuning, underscoring that online multiplier adaptation is essential for stable and controllable length reduction.

\begin{table*}[h]
\centering
\caption{Evaluation results of fixed $\lambda$ ablation across four math benchmarks.}
\renewcommand{\arraystretch}{1.2}
\resizebox{\textwidth}{!}{%
\begin{tabular}{lcccccccccc}
\toprule
\multirow{2}{*}{\textbf{Model}} &
\multicolumn{2}{c}{\textbf{AIME2024}} & 
\multicolumn{2}{c}{\textbf{MATH}} & 
\multicolumn{2}{c}{\textbf{Minerva}} & 
\multicolumn{2}{c}{\textbf{Olympiad}} & 
\multicolumn{2}{c}{\textbf{Macro Average}} \\
\cmidrule(lr){2-3} \cmidrule(lr){4-5} \cmidrule(lr){6-7} \cmidrule(lr){8-9} \cmidrule(lr){10-11}
& Pass@1 & \# Tokens & Pass@1 & \# Tokens & Pass@1 & \# Tokens & Pass@1 & \# Tokens & Pass@1 & \# Tokens \\
\midrule
\textrm{DeepScaleR-1.5B} & 
38.75 & 8140 & 
85.92 & 3019 & 
27.62 & 4134 & 
51.94 & 5410 & 
51.06 & 5176 \\
\;\textrm{+ GRPO $(\lambda=0)$} & 
29.58 & 6122 &
85.28 & 1767 & 
29.64 & 2630 & 
49.11 & 3418 & 
\cellann{48.40}{\RedDown{2.66}} & \cellann{3484}{\GreenDown{33\%}} \\
\;\textrm{+ $\lambda=0.001$} & 
27.92 & 5834 &
82.73 & 2089 &
29.64 & 2200 &
46.76 & 3663 &
\cellann{46.76}{\RedDown{4.3}} & \cellann{3396}{\GreenDown{34\%}} \\
\;\textrm{+ $\lambda=0.01$} &
26.25 & 4709 &
81.65 & 1927 &
28.90 & 1990 &
45.52 & 2963 &
\cellann{45.58}{\RedDown{5.48}} & \cellann{2897}{\GreenDown{44\%}} \\
\;\textrm{+ $\lambda=0.05$} & 
21.66 & 2417 &
78.28 & 1018 &
24.92 & 1073 &
41.25 & 1602 &
\cellann{41.53}{\RedDown{9.53}} & \cellann{1528}{\GreenDown{70\%}} \\
\;\textbf{+ \texttt{LACONIC}} & 
28.12 & 2665 &
83.75 & 1049 &
29.24 & 1189 &
47.07 & 1669 &
\cellann{47.05}{\RedDown{4.01}} & \cellann{1643}{\GreenDown{68\%}} \\
\bottomrule
\end{tabular}
}
\vspace{0.5em}
\label{tab:ablation-results-fixed-lambda}
\end{table*}

\section{Analysis of clipped-cost primal-dual}
\label{app:clipped-convergence}

We first restate some definitions and notations. Let $q\sim P(Q)$ denote a prompt and $o\sim \pi(\cdot\mid q)$ be a response.
Let $L(o)$ be the length of response $o$ and $B>0$ be the token budget.
Define
\[
R(\pi):=\EE_{q\sim P(Q),o\sim \pi(\cdot|q)}[r(q,o)],\quad
\tilde C(\pi):=\EE_{q\sim P(Q),o\sim \pi(\cdot|q)}\!\left[\frac{L(o)-B}{B}\right],\quad
C(\pi):=\EE_{q\sim P(Q),o\sim \pi(\cdot|q)}\!\left[\frac{(L(o)-B)_+}{B}\right],
\]
where $(x)_+ = \max\{x,0\}$.
The target constrained problem is
\[
\pi^\star \in \argmax_{\pi} R(\pi)\quad \text{s.t.}\quad \tilde C(\pi)\le 0.
\]

\subsection{Idealized clipped-cost dynamics}
For any $\lambda\ge 0$, the optimal policy model under clipped-cost corresponding to $\lambda$ can be determined and defined as 
\begin{equation}
    \pi(\lambda)\in\argmax_{\pi} R(\pi)-\lambda\, C(\pi),
\end{equation}
and then we define the induced true normalized length of the resulting 
\begin{equation}
    \mu(\lambda):=\tilde C(\pi(\lambda)).
\end{equation}
We study the deterministic projected update
\begin{equation}
    \pi_t=\pi(\lambda_t),\quad
    \lambda_{t+1}=\textrm{clip}(\lambda_t + \eta\,\mu(\lambda_t), 0, \Lambda),
\end{equation}
and assume it converges to a limit $(\pi^\sharp,\lambda^\sharp)$ with $\pi^\sharp=\pi(\lambda^\sharp)$.

\subsection{Bounding Reward Suboptimality of Clipped-Cost}

\begin{lemma}\label{lem:upper-bound-of-C(pi)}
    Denote the response length cap by $L_{\max}$, and $0\le L(o)\le L_{\max}$ for any response $o$. We also require $L_{\max}>B$, otherwise the problem is trivial. Let $\pi$ be any feasible policy for the linear constraint $\tilde{C}(\pi)\le 0$, i.e.,
    $\EE_{q\sim P(Q),o\sim \pi(\cdot|q)}[L(o)]\le B$. Then 
    \begin{equation}\label{eq:upper-bound-of-C(pi)}
        C(\pi) = \EE_{q\sim P(Q),o\sim \pi(\cdot|q)}[\frac{(L(o)- B)_+}{B}] \le \frac{L_{\max} - B}{L_{\max}}.
    \end{equation}
\end{lemma}

\begin{proof}
    We first notice that for any response $o$, 
    \begin{equation}\label{eq:inequality}
        (L(o)-B)_+\le \frac{L_{\max} - B}{L_{\max}}L(o).
    \end{equation}
    When $L(o)< B$, then $\textrm{LHS}<0<\textrm{RHS}$. When $B\le L(o) \le L_{\max}$, then 
    $$(L(o)-B)_+ = L(o) - B \le \frac{L_{\max} - B}{L_{\max}}L(o),$$
    where the inequality can be verified by arithmetic rearranging.
    
    Then by taking expecations on both sides of \cref{eq:inequality}, we have
    $$\EE_{q\sim P(Q),o\sim \pi(\cdot|q)}[(L(o) - B)_+]\le \frac{L_{\max} - B}{L_{\max}}\EE_{q\sim P(Q),o\sim \pi(\cdot|q)}[L(o)].$$
    Since $\EE_{q\sim P(Q),o\sim \pi(\cdot|q)}[L(o)]\le B$ for the feasible policy $\pi$, we conclude that 
    $$C(\pi) = \EE_{q\sim P(Q),o\sim \pi(\cdot|q)}[\frac{(L(o)- B)_+}{B}] \le \frac{L_{\max} - B}{L_{\max}}.$$
\end{proof}

\begin{theorem}[Restate of \cref{thm:price_of_clipping}.] Let $\pi^\star\in\max_{\pi:\tilde{C}(\pi)\le 0} R(\pi)$ be an optimal feasible policy of the length-constrained problem in \cref{eq:CMDP}. Let $(\pi^\sharp,\lambda^\sharp)$ be the feasible limit of the idealized clipped-cost primal-dual updates in \cref{eq:ideal_dyn_primal,eq:ideal_dyn_dual}. Then
\begin{equation}\label{eq:price_of_clipping_prelim_restate}
0 \le R(\pi^\star)-R(\pi^\sharp)\le \lambda^\sharp C(\pi^*).
\end{equation}
Moreover, for indicator rewards with the $\lambda$-ceiling $\Lambda = \frac{B}{L_{\max} - B}$ and a maximum length cap $L(o)\le L_{\max}$, we have
\begin{equation}\label{eq:price_of_clipping_restate}
0 \le R(\pi^\star)-R(\pi^\sharp)\le \frac{B}{L_{\max}}.
\end{equation}
\end{theorem}
\begin{proof}
    Since $\tilde{C}(\pi^\sharp)\le 0$, $\pi^\sharp$ is feasible for the clipped-cost primal-dual optimization, and therefore feasible for the original constrained optimization, by optimality of $\pi^*$,
    $$R(\pi^*) \ge R(\pi^\sharp),$$
    which gives the left inequality.

    For the upper bound, the optimal feasible policy $\pi^\sharp$ implies that for any policy $\pi$,
    $$R(\pi^\sharp) - \lambda^\sharp C(\pi^\sharp) \ge R(\pi) - \lambda^\sharp C(\pi).$$
    Let $\pi$ be $\pi^*$ and rearranging,
    $$R(\pi^*) - R(\pi^\sharp)\le \lambda^\sharp C(\pi^*) - \lambda^\sharp C(\pi^\sharp).$$
    Since $C(\pi^\sharp) \ge 0$, dropping the negative term, we conclude that 
    $$R(\pi^*) - R(\pi^\sharp) \le \lambda^\sharp C(\pi^*).$$
    Thus, we have proved \cref{eq:price_of_clipping_prelim_restate}. Finally, since 
    $$\lambda^\sharp\le \Lambda \le \frac{B}{L_{\max} - B},$$
    and since $\pi^*$ is feasible, we apply \cref{lem:upper-bound-of-C(pi)},
    and obtain \cref{eq:price_of_clipping_restate}:
    $$R(\pi^*) - R(\pi^\sharp) \le \lambda^\sharp C(\pi^*) \le \frac{B}{L_{\max} - B}\frac{L_{\max} - B}{L_{\max}} = \frac{B}{L_{\max}}.$$
\end{proof}

\subsection{Convergence Guarantee}
We impose following assumptions on the scalar response function $\mu(\lambda) = \tilde{C}(\pi(\lambda))$.

\textbf{A1.} For any $\lambda\in [0,\Lambda]$, $\pi(\lambda)$ exists.

\textbf{A2 (Monotonicity of $\mu$).} The response function $\mu$ is nonincreasing and $\mu(\Lambda)\le 0$. 

The mononicity assumption states that as the dual variable $\lambda$ increases, the corresponding optimal policy that maximizes $R(\pi)-\lambda C(\pi)$ will not generate longer responses in expectation. Also, the budget can be enforced without needing $\lambda>\Lambda$.

\textbf{A3 (Continuity of $\mu$).} The response function $\mu$ is continuous and $\gamma$-lipschitz on $[0,\Lambda]$, i.e., for any $\lambda,\lambda'\in [0,\Lambda]$:
$$\abs{\mu(\lambda) - \mu(\lambda')} \le \gamma \abs{\lambda - \lambda'}.$$

\begin{lemma}[Fixed-point feasibility]\label[lemma]{lem:fixed_point}
    Let $\lambda^\sharp\in [0,\Lambda]$ be the fixed point of the dual update
    $$\lambda^\sharp = \textrm{clip}(\lambda^\sharp + \eta\, \mu(\lambda^\sharp), 0, \Lambda).$$
    Then $\mu(\lambda^\sharp)\le 0$, i.e., $\tilde{C}(\pi(\lambda^\sharp)) \le 0$. Hence, $\pi^\sharp = \pi(\lambda^\sharp)$ is feasible.
\end{lemma}
\begin{proof}
    We discuss the value of $\lambda^\sharp$: (i) if $0<\lambda^\sharp<\Lambda$, then it implies $\eta\,\mu(\lambda^\sharp) = 0$, and thus $\mu(\lambda^\sharp) = 0$; (ii) if $\lambda^\sharp = 0$, then $\eta\,\mu(\lambda^\sharp) \le 0$, and thus $\mu(\lambda^\sharp)\le 0$; (iii) if $\lambda^\sharp = \Lambda$, then by assumption A2, we have $\mu(\Lambda) \le 0$.
\end{proof}

\begin{theorem}[Convergence]\label{thm:convergence}
    Let the step size $0< \eta\le 1/\gamma$. Then for any initialization $\lambda_0\in [0,\Lambda]$, the sequence $\lambda_{t+1} = \textrm{clip}(\lambda_t + \eta\, \mu(\lambda_t), 0, \Lambda)$ converges to a fixed point $\lambda^\sharp\in [0,\Lambda]$. Moreover, the corresponding policy $\pi^\sharp = \pi(\lambda^\sharp)$ is feasible, i.e., $\tilde{C}(\pi^\sharp)\le 0$.
\end{theorem}

\begin{proof}
    We turn the problem into convex optimization. First, define the potential 
    $$\Phi(\lambda) = \int_0^\lambda \mu(s)ds.$$
    By the assumption A1 that $\mu$ is nonincreasing, $\Phi$ is concave, and by the assumption A3, $\Phi$ is $\gamma$-smooth. Then we define a convex objective 
    $$f(\lambda) = -\Phi(\lambda).$$
    We then have 
    $$\nabla f(\lambda) = -\mu(\lambda),$$
    and the dual updates can be written as
    $$\lambda_{t+1} = \textrm{clip}(\lambda_t - \eta\nabla f(\lambda_t), 0, \Lambda).$$
    By Fejer monotonicity (Theorem 10.23 and 10.24 in \cite{doi:10.1137/1.9781611974997}), we claim that the distance between the iterates $\lambda_t$ and the optimal solution (i.e., the fixed point) $\lambda^\sharp$ is nonincreasing:
    $$\norm{\lambda_{t+1} - \lambda^\sharp}^2\le \norm{\lambda_{t} - \lambda^\sharp}^2,$$
    and $\lambda_t$ converges to the fixed point $\lambda^\sharp$. Finally, by \cref{lem:fixed_point}, the corresponding policy $\pi^\sharp = \pi(\lambda^\sharp)$ is feasible.
\end{proof}

\subsection{Convergence Rate}
For a clean linear convergence rate, we impose an additional assumption on $\mu$.

\textbf{A4 (Strong monotonicity).} There exists $\xi>0$ such that for all $\lambda\ge \lambda'\in [0,\Lambda]$, $\mu(\lambda) - \mu(\lambda') \le -\xi (\lambda - \lambda')$.

Under the assumptions A1-A4, there exists a unique $\lambda^*\in [0, \Lambda]$ such that $\mu(\lambda^*) = 0$.

\begin{theorem}[Convergence rate]\label{thm:convergence-rate}
    Let the step size $0<\eta\le \min\{1/\gamma, 1/\xi\}$. Then the dual iterates $\lambda_t$ satisfy
    \begin{equation}\label{eq:contraction}
        \abs{\lambda_t - \lambda^\sharp} \le (1-\eta \xi)^t\abs{\lambda_0 - \lambda^\sharp}.
    \end{equation}
    Moreover, the constraint violation decays
    \begin{equation}\label{eq:convergence-rate}
        \abs{\mu(\lambda_t)} \le \gamma \abs{\lambda_t - \lambda^\sharp} \le \gamma(1-\eta\xi)^t \abs{\lambda_0 - \lambda^\sharp},
    \end{equation}
    which implies that $\tilde{C}(\pi_t) = \mu(\lambda_t)$ converges to 0 at a geometric rate.
\end{theorem}
\begin{proof}
    We denote the dual update by a mapping operator $T(\lambda)$:
    $$T(\lambda) = \textrm{clip}(\lambda + \eta \,\mu(\lambda),0,\Lambda).$$
    Since $\lambda^\sharp\in [0,\Lambda]$ and $\mu(\lambda^\sharp) = 0$,
    $$T(\lambda^\sharp) = \lambda^\sharp.$$
    Since the clipping function is non-expansive, we have
    \begin{equation}\label{eq:int-eq-in-conv-rate}
        \abs{T(\lambda) - T(\lambda^\sharp)}\le \abs{\lambda - \lambda^\sharp + \eta (\mu(\lambda) - \mu(\lambda^\sharp))} = \abs{\lambda - \lambda^\sharp + \eta \,\mu(\lambda)}.
    \end{equation}
    Then by the assumption A5, if $\lambda \ge \lambda^\sharp$,
    $$\mu(\lambda)\le \mu(\lambda^\sharp) - \xi(\lambda - \lambda^\sharp) = -\xi(\lambda - \lambda^\sharp).$$
    Therefore, 
    $$\lambda - \lambda^\sharp + \eta\,\mu(\lambda)\le (\lambda - \lambda^\sharp) -\eta\xi(\lambda - \lambda^\sharp) = (1-\eta\xi)(\lambda - \lambda^\sharp),$$
    and 
    $$\abs{\lambda - \lambda^\sharp + \eta\,\mu(\lambda)}\le (1-\eta\xi)(\lambda - \lambda^\sharp).$$
    The same bound follows for $\lambda\le \lambda^\sharp$. Plugging back in \cref{eq:int-eq-in-conv-rate}, we have
    $$\abs{T(\lambda) - T(\lambda^\sharp)} \le (1-\eta\xi)(\lambda - \lambda^\sharp).$$
    Let $\lambda = \lambda_t$, and note that $T(\lambda_t) = \lambda_{t+1}$, we obtain the contraction of the dual update operator $T$:
    $$\abs{\lambda_{t+1} - \lambda^\sharp}\le (1-\eta\xi)\abs{\lambda_t - \lambda^\sharp}.$$
    Telescoping gives \cref{eq:contraction}. Finally, notice that $\mu(\lambda^\sharp) = 0$, we have 
    $$\abs{\mu(\lambda_t)} = \abs{\mu(\lambda_t) - \mu(\lambda^\sharp)}\le \gamma \abs{\lambda_{t+1} - \lambda^\sharp},$$
    and plugging in \cref{eq:contraction} gives us \cref{eq:convergence-rate}.
\end{proof}

\section{Case Study}

Prompt: Every morning Aya goes for a $9$-kilometer-long walk and stops at a coffee shop afterwards. When she walks at a constant speed of $s$ kilometers per hour, the walk takes her 4 hours, including $t$ minutes spent in the coffee shop. When she walks $s+2$ kilometers per hour, the walk takes her 2 hours and 24 minutes, including $t$ minutes spent in the coffee shop. Suppose Aya walks at $s+\frac{1}{2}$ kilometers per hour. Find the number of minutes the walk takes her, including the $t$ minutes spent in the coffee shop. You FIRST think about the reasoning process as an internal monologue and then provide the final answer. The reasoning process MUST BE enclosed within <think> </think> tags. The final answer MUST BE put in \textbackslash boxed\{\}.

Ground truth answer: $204$.

\begin{tcolorbox}[modelbox={\texttt{LACONIC} Output (\#Tokens = 855)}]
\begin{flushleft}
Alright, let's tackle this problem step by step.

First, let's understand what's given. Aya walks a total distance of 9 kilometers. The time she takes depends on her walking speed. She stops at a coffee shop, and the time spent there is \(t\) minutes.

We are given two scenarios:

1. When her speed is \(s\) km/h, the total time (walking plus coffee) is 4 hours.  
2. When her speed increases to \(s+2\) km/h, the total time becomes 2 hours and 24 minutes, which is \(2 + \frac{24}{60} = 2.4\) hours.

We need to find the total time when her speed is \(s + 0.5\) km/h.

Let's denote the time spent in the coffee shop as \(t\) minutes. Since the coffee shop time is the same in both scenarios, we can set up equations based on the total time.

In the first scenario:  
Total time \(=\) walking time \(+\) \(t\) minutes.  
Walking time \(=\) distance / speed \(= 9 / s\) hours.  
Therefore, \(9/s + t = 4\) hours.

In the second scenario:  
Walking time \(=\) \(9 / (s + 2)\) hours.  
Total time \(=\) \(9/(s+2) + t = 2.4\) hours.

Now, we have two equations:  
1. \(9/s + t = 4\)  
2. \(9/(s+2) + t = 2.4\)

We can subtract equation 2 from equation 1 to eliminate \(t\):  
\((9/s + t) - (9/(s+2) + t) = 4 - 2.4\)

Simplify:  
\(9/s - 9/(s+2) = 1.6\)

Factor out 9:  
\(9[1/s - 1/(s+2)] = 1.6\)

Compute the difference inside the brackets:  
\(1/s - 1/(s+2) = (s+2 - s)/(s(s+2)) = 2/(s(s+2))\)  
So, \(9*(2)/(s(s+2)) = 1.6\)

Simplify:  
\(18/(s(s+2)) = 1.6\)

Multiply both sides by \(s(s+2)\):  
\(18 = 1.6\, s(s+2)\)

Divide both sides by 1.6:  
\(s(s+2) = 18 / 1.6\)

Calculate \(18 / 1.6\):  
\(18 / 1.6 = 11.25\)

So, \(s^2 + 2s - 11.25 = 0\)

Now, solve the quadratic equation:  
\(s^2 + 2s - 11.25 = 0\)

Using quadratic formula:  
\(s = [-2 \pm \sqrt{4 + 45}] / 2\)

Since speed can't be negative:  
\(s = [-2 + \sqrt{49}] / 2\)  
\(\sqrt{49} = 7\)

Thus, \(s = (-2 +7)/2 = 5/2 = 2.5\) km/h

Now, we can find \(t\) from equation 1:  
\(9/s + t = 4\)  
\(9/2.5 + t = 4\)

Compute \(9/2.5 = 3.6\)  
So, \(3.6 + t = 4\)

Thus, \(t = 0.4\) hours \(= 0.4*60 = 24\) minutes.

Now, when her speed is \(s + 0.5 = 2.5 + 0.5 = 3\) km/h.

We need to find the total time including coffee shop.

First, compute walking time:  
\(9 / 3 = 3\) hours.

Total time \(=\) walking time \(+\) coffee shop time \(= 3 + 0.4 = 3.4\) hours.

Convert to minutes: \(3.4*60 = 204\) minutes.

So, the total time is 204 minutes.

\textless/think\textgreater

The total time Aya takes when walking at \( s + \frac{1}{2} \) km/h is \(\boxed{204}\) minutes.
\end{flushleft}

\end{tcolorbox}

\begin{tcolorbox}[modelbox={GRPO (\#Tokens = 1763)}]
\begin{flushleft}
Okay, so I have this problem about Aya's walking and coffee shop time. Let me try to figure it out step by step. \\

First, the problem says that Aya walks 9 kilometers every morning and stops at a coffee shop. When she walks at a speed of $s$ km/h, it takes her 4 hours, including $t$ minutes spent in the coffee shop. Then, when she increases her speed to $s+2$ km/h, the walk takes her 2 hours and 24 minutes, but this time includes the same $t$ minutes in the coffee shop. \\

I need to find the time it takes for Aya to walk at a speed of $s + \tfrac12$ km/h, including the $t$ minutes in the coffee shop. \\

Let me break down the information given: \\

\begin{enumerate}
\item When walking at speed $s$ km/h, the total time taken is 4 hours, including $t$ minutes in the coffee shop. So, the walking time is 4 hours minus $t$ minutes. But since speed is in km/h, maybe it's better to convert everything into hours.
\item When walking at speed $s + 2$ km/h, the total time is 2 hours and 24 minutes, which is $2 + \frac{24}{60} = 2.4$ hours. Again, the walking time is total time minus $t$ minutes.
\end{enumerate}

So, let me denote: \\

\begin{itemize}
\item Walking speed: $s$ km/h
\item Walking speed in the second case: $s + 2$ km/h
\item Total time in first case: 4 hours
\item Total time in second case: 2.4 hours
\item Both cases include $t$ minutes in the coffee shop.
\end{itemize}

We need to find the total time when walking speed is $s + \tfrac12$ km/h. \\

First, let's convert the walking times into hours to make calculations easier. \\

For the first scenario: \\

Total time $= 4$ hours \\
Time spent walking $= 4 - \tfrac{t}{60}$ hours (since $t$ is in minutes) \\

Distance walked is 9 km, so using the formula: \\

Distance $=$ speed $\times$ time \\

So, $9 = s \times \left(4 - \tfrac{t}{60}\right)$ \\

Similarly, for the second scenario: \\

Total time $= 2.4$ hours \\
Time spent walking $= 2.4 - \tfrac{t}{60}$ hours \\

Distance walked is 9 km again: \\

$9 = (s + 2) \times \left(2.4 - \tfrac{t}{60}\right)$ \\

Now, we have two equations: \\

1) $9 = s \times \left(4 - \tfrac{t}{60}\right)$ \\
2) $9 = (s + 2) \times \left(2.4 - \tfrac{t}{60}\right)$ \\

Since both equal 9, we can set them equal to each other: \\

$s \times \left(4 - \tfrac{t}{60}\right) = (s + 2) \times \left(2.4 - \tfrac{t}{60}\right)$ \\

Let me expand both sides: \\

Left side: $4s - s\left(\tfrac{t}{60}\right)$ \\
Right side: $(s + 2)\left(2.4 - \tfrac{t}{60}\right) = 2.4s - s\left(\tfrac{t}{60}\right) + 4.8 - 2\left(\tfrac{t}{60}\right)$ \\

Simplify both sides: \\

Left: $4s - \dfrac{s t}{60}$ \\
Right: $2.4s - \dfrac{s t}{60} + 4.8 - \dfrac{2 t}{60}$ \\

Now, let's write the equation: \\

$4s - \dfrac{s t}{60} = 2.4s - \dfrac{s t}{60} + 4.8 - \dfrac{2 t}{60}$ \\

Notice that the $-\dfrac{s t}{60}$ terms on both sides cancel each other out. So we can subtract $2.4s$ from both sides: \\

$4s - 2.4s = - \dfrac{s t}{60} + \dfrac{s t}{60} + 4.8 - \dfrac{2 t}{60}$ \\

Simplify: \\

$1.6s = 4.8 - \dfrac{2 t}{60}$ \\

Simplify $1.6s = 4.8 - \dfrac{t}{30}$ \\

Let me write that as: \\

$1.6s + \dfrac{t}{30} = 4.8$ \\

\textbf{Equation (A):} $1.6s + \dfrac{t}{30} = 4.8$ \\

Now, let's go back to the first equation: \\

$9 = s \times \left(4 - \tfrac{t}{60}\right)$ \\

Let me write this as: \\

$9 = 4s - \dfrac{s t}{60}$ \\

Multiply both sides by 60 to eliminate denominators: \\

$9 \times 60 = 4s \times 60 - s t$ \\

$540 = 240s - s t$ \\

Let's write this as: \\

$s t = 240s - 540$ \\

Divide both sides by $s$ (assuming $s \neq 0$): \\

$t = 240 - \dfrac{540}{s}$ \\

\textbf{Equation (B):} $t = 240 - \dfrac{540}{s}$ \\

Now, substitute Equation (B) into Equation (A): \\

$1.6s + \dfrac{240 - \tfrac{540}{s}}{30} = 4.8$ \\

Simplify: \\

$1.6s + \dfrac{240}{30} - \dfrac{540}{30s} = 4.8$ \\

Calculate: \\

$1.6s + 8 - \dfrac{18}{s} = 4.8$ \\

Now, subtract 4.8 from both sides: \\

$1.6s + 8 - \dfrac{18}{s} - 4.8 = 0$ \\

Simplify: \\

$1.6s + 3.2 - \dfrac{18}{s} = 0$ \\

Multiply both sides by $s$ to eliminate denominator: \\

$1.6s^{2} + 3.2s - 18 = 0$ \\

Now, we have a quadratic equation: \\

$1.6s^{2} + 3.2s - 18 = 0$ \\

Let me multiply both sides by 10 to eliminate decimals: \\

$16s^{2} + 32s - 180 = 0$ \\

Divide all terms by 4 to simplify: \\

$4s^{2} + 8s - 45 = 0$ \\

Now, solve for $s$ using quadratic formula: \\

$s = \dfrac{-b \pm \sqrt{b^{2} - 4ac}}{2a}$ \\

Where $a = 4$, $b = 8$, $c = -45$ \\

Calculate discriminant: \\

$D = 8^{2} - 4 \times 4 \times (-45) = 64 + 720 = 784$ \\

Square root of $D$: $\sqrt{784} = 28$ \\

So, \\

$s = \dfrac{-8 \pm 28}{2 \times 4} = \dfrac{-8 \pm 28}{8}$ \\

We have two solutions: \\

1) $s = \dfrac{-8 + 28}{8} = \dfrac{20}{8} = 2.5$ km/h \\

2) $s = \dfrac{-8 - 28}{8} = \dfrac{-36}{8} = -4.5$ km/h \\

Since speed cannot be negative, $s = 2.5$ km/h \\

Now, find $t$ using Equation (B): \\

$t = 240 - \dfrac{540}{s} = 240 - \dfrac{540}{2.5}$ \\

Calculate $\dfrac{540}{2.5}$: \\

$540 / 2.5 = 216$ \\

So, \\

$t = 240 - 216 = 24$ minutes \\

So, $t = 24$ minutes. \\

Now, when Aya walks at $s + \tfrac12$ km/h, which is $2.5 + 0.5 = 3$ km/h. \\

We need to find the total time including $t$ minutes in the coffee shop. \\

First, calculate the walking time at 3 km/h. \\

Distance is 9 km, speed is 3 km/h, so time is $\dfrac{9}{3} = 3$ hours. \\

But total time is walking time plus $t$ minutes. \\

Walking time is 3 hours, which is 180 minutes. \\

Plus $t = 24$ minutes. \\

So total time is $180 + 24 = 204$ minutes. \\

Therefore, the walk takes her 204 minutes, including the $t$ minutes spent in the coffee shop. \\

\textless/think\textgreater \\

The number of minutes the walk takes her, including the $t$ minutes spent in the coffee shop, is $\boxed{204}$. 
\end{flushleft}

\end{tcolorbox}


\end{document}